\documentclass{article}

\usepackage{arxiv}

\usepackage[utf8]{inputenc} 
\usepackage[T1]{fontenc}    
\usepackage{hyperref}       
\usepackage{url}            
\usepackage{booktabs}       
\usepackage{amsfonts}       
\usepackage{nicefrac}       
\usepackage{microtype}      
\usepackage{lipsum}
\usepackage{biblatex}       
\usepackage{tabularx}
\usepackage{graphicx}
\usepackage{multirow}
\usepackage{pdfpages}

\usepackage{amsthm}  
\newtheorem{proposition}{Proposition}
\newtheorem*{propositionnonumber}{Proposition}

\usepackage[
  separate-uncertainty = true,
  multi-part-units = repeat
]{siunitx}

\usepackage{color}

\definecolor{orange}{rgb}{1.0, 0.5, 0.0}
\definecolor{violet}{rgb}{1.0, 0, 1.0}
\definecolor{darkgreen}{rgb}{0, 0.5, 0}
\definecolor{darkmagenta}{rgb}{0.55, 0.0, 0.55}

\newcommand{\sidenoteJose}[1]{\textcolor{violet}{${\leftarrow}\hspace{0pt}{\bullet}$}\marginpar{\textcolor{violet}{${\leftarrow}\hspace{0pt}{\bullet}$}\\\tiny{\textcolor{violet}{Jose: #1}}}}

\newcommand{\comment}[1]{}

\newcommand{\todo}[1]{\textcolor{orange}{[TODO: #1]}}

\usepackage{xspace}

\newcommand{\xaxis}{$x$-axis\xspace}
\newcommand{\yaxis}{$y$-axis\xspace}

\newcommand{\Acc}{\ensuremath{\mathit{Acc}}\xspace}
\newcommand{\SPD}{\ensuremath{\mathit{SPD}}\xspace}

\newcommand{\TP}{\ensuremath{\mathit{TP}}\xspace}
\newcommand{\TN}{\ensuremath{\mathit{TN}}\xspace}
\newcommand{\FP}{\ensuremath{\mathit{FP}}\xspace}

\newcommand{\Pos}{\ensuremath{\mathit{Pos}}\xspace}
\newcommand{\Neg}{\ensuremath{\mathit{Neg}}\xspace}

\newcommand{\missing}{\ensuremath{\odot}}

\newcommand{\amplifies}{\ensuremath{^\bigtriangleup}}
\newcommand{\reduces}{\ensuremath{^\bigtriangledown}}

\usepackage{wasysym}

\title{Fairness and Missing Values}

\author{
Fernando Martínez-Plumed \And Cèsar Ferri \And David Nieves \And José Hernández-Orallo \\
\\
Valencian Research Institute for Artificial Intelligence (VRAIN)\\
Universitat Politècnica de València, Spain\\
\texttt{\{fmartinez,daniecor,cferri,jorallo\}@dsic.upv.es} \\
}

\begin{document}
\maketitle

\begin{abstract}
The causes underlying unfair decision making are complex, being internalised in different ways by decision makers, other actors dealing with data and models, and ultimately by the individuals being affected by these decisions. One frequent manifestation of all these latent causes arises in the form of missing values: protected groups are more reluctant to give information that could be used against them, delicate information for some groups can be erased by human operators, or data acquisition may simply be less complete and systematic for minority groups. As a result, missing values and bias in data are two phenomena that are tightly coupled. However, most recent techniques, libraries and experimental results dealing with fairness in machine learning have simply ignored missing data. In this paper, we claim that fairness research should not miss the opportunity to deal properly with missing data. To support this claim, (1) we analyse the sources of missing data and bias, and we map the common causes, (2) we find that rows containing missing values are usually fairer than the rest, which should not be treated as the uncomfortable ugly data that different techniques and libraries get rid of at the first occasion, and (3) we study the  trade-off between performance and fairness when the rows with missing values are used (either because the technique deals with them directly or by imputation methods). We end the paper with a series of recommended procedures about what to do with missing data when aiming for fair decision making. 
\end{abstract}


\section{Introduction}
Because of the ubiquitous use of machine learning and artificial intelligence for decision making, there is an increasing urge in ensuring that these algorithmic decisions are fair, i.e., they do not discriminate some groups over others, especially with groups that are defined over protected attributes such as gender, race and nationality  \cite{dwork2012fairness,angwin2016machine,executive2016big,noble2018algorithms,grgic2018human,speicher2018unified}. 
Despite all this growing research interest, fairness in decision making did not arise as a consequence of the use of machine learning and other predictive models in data science and AI \cite{ethical_societal}. Fairness is an old and fundamental concept when dealing with data that should cover all data processing activities, from data gathering to data cleansing, through modelling and model deployment. It is not simply that data is biased and this can be amplified by algorithms, but rather that data processing can introduce more bias, from data collection procedures to model deployment \cite{berk2017convex,bolukbasi2016man,hardt2016equality,NIPS2017_6995,zafar2017parity}.

It is therefore no surprise that fairness strongly depends on both the quality of the data and the quality of the processing of these data \cite{barocas2016big}. One major issue for data quality is the presence of missing data, which may represent absence of information but also some information that has been removed due to several possible reasons (inconsistency, privacy, or other interventions). Once missing data appears in the pipeline it becomes an {\em ugly duckling} for many subsequent processes, such as data visualisation and summarisation, feature selection and engineering, and model construction and deployment. It is quite common that missing values are removed or replaced as early as possible, so that they no longer become a nuisance for a bevy of theoretical methods and practical tools.

In this context, we ask the following questions. Are  missing data and fairness related? Are those subsamples with missing data more or less unfair? Is it the right procedure to delete or replace these values as many theoretical models and machine learning libraries do by default? In this paper we analyse all these questions for the first time and give a series of recommendations about how to proceed with missing values if we are (as we should be) concerned by unfair decisions.

Let us illustrate these questions with an example. The Adult Census Data \cite{kohavi2001data} is one of the most frequently used datasets in the fairness literature, where race and sex are attributes that could be used to define the protected groups. 
 Adult has 48,842 records and a binary label indicating a salary of $\leq$\$50K or $>$\$50K. There are 14 attributes: 8 categorical and 6 continuous attributes. The prediction task is to determine whether a person makes over 50K a year based on their attributes.  
Not surprisingly, as we see in Figure
\ref{fig:AdultNA}, there is an important number of missing values, as it is usually the case for real-world datasets, and most especially those dealing with personal data. As we will see later, the {\em missingness distribution} in this dataset is {\em not} 
missing completely at random (MCAR). This means that discarding or modifying the rows with missing values can bias the sample. But, more interestingly, these missing values appear in the {\tt occupation}, {\tt workclass} and {\tt native.country} attributes, which seem to be strongly related to the protected attributes. As a result, the bias that is introduced by discarding or modifying the rows with missing values can have an important effect on fairness.

\begin{figure}[!htb]
        \center{\includegraphics[width=0.55\textwidth]
        {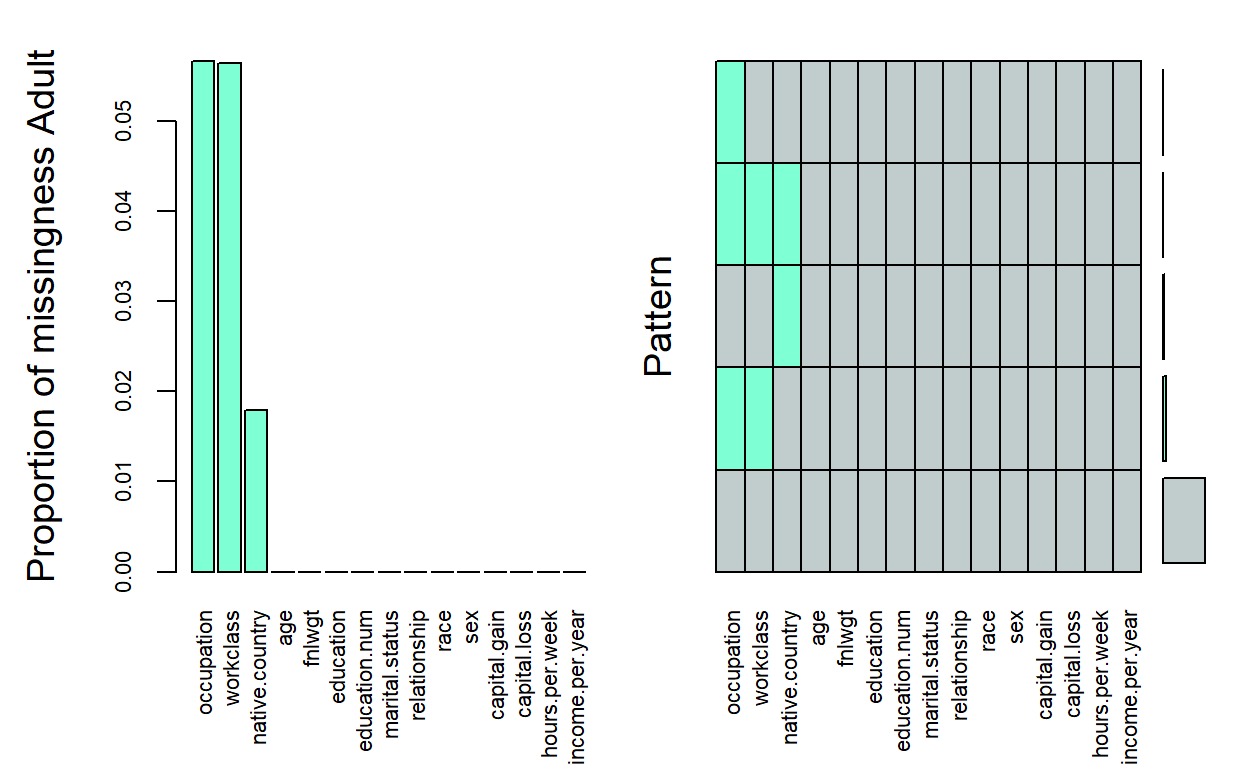}}
        \caption{\label{fig:AdultNA} \emph{(Left)} Histogram (\%) of missing values per input variable in the Adult dataset. Missing values are concentrated in three of the categorical features:  \texttt{workclass} (6\%), \texttt{native.country} (2\%), and \texttt{occupation} (6\%).  \emph{(Right)} Aggregation plot of all existing combinations of missing and non-missing values in the dataset. It is unlikely that missingness in the Adult dataset is MCAR because many missing values for the attribute \texttt{workclass} are also missing for \texttt{occupation} (about 6\% of examples have missing values in both).}
\end{figure}
    
On one hand, as missing values are so commonplace, there are
many techniques for data cleansing, feature selection and model optimisation that have been designed to convert (or get rid of) missing data with the aim of improving some performance metrics. However, --to our knowledge--, fairness has never been considered in any of these techniques.

On the other hand, many theoretical methods for dealing with fairness --the so-called mitigation methods-- do not mention missing values at all. Conditional probabilities, mutual information and other distributional concepts get ugly when attributes can have a percentage of missing values. Much worse, when the theory is brought to practice and ultimately to implementation, we see that the most common libraries (AIF360 toolkit \cite{aif360-oct-2018}, Aequitas \cite{2018aequitas}, Themis-ML \cite{bantilan2018themis}, Fairness-Comparison library \cite{friedler2019comparative}) 
simply remove the rows or columns containing the missing values, or assume the datasets have been preprocessed before being analysed (removing missing values), otherwise throwing an error to the user. As a result, the literature using these techniques and libraries simply report results about fairness as if the datasets did not contain the rows with the missing data. In other words, the datasets are simply mutilated.

Apart from the Adult dataset, we will also explore two other datasets that have potential fairness issues and, as usual, contain missing values.  The Recidivism dataset contains variables used by the COMPAS algorithm  to assess potential recidivism risk, where the class indicates whether the inmate commits a crime in less than two years after being released from prison or not. Data for over 10,000 criminal defendants in Florida was gathered by ProPublica \cite{angwin2016machine}, showing that black defendants were often predicted to be at a higher risk of recidivism than they actually were, compared to their white counterparts. 
In this datasets, there are also three attributes with missing values:  \texttt{days\_b\_screening\_arrest}, \texttt{c\_days\_from\_compas} and \texttt{c\_charge\_desc}.
Finally, we also study the data for 891 of the real Titanic passengers, a case where bias is more explicit than in the other cases. The class represents whether the passenger survived or not, where the conditional probability that a person survives given their \texttt{sex} and \texttt{passenger-class} is higher for females from higher classes. There are also three attributes with missing values:  \texttt{age}, \texttt{fare} and \texttt{embarked}. The distribution of these missing values is not MCAR either for these two datasets, and because of their semantics, it seems they can have an effect on fairness. 

In what follows 
 we overview the reasons why missing values appear, what kinds of missingness there are and how missing values are handled (e.g., imputation). Next, in section \ref{sec:fairness} we analyse the causes of fairness, along with some metrics, mitigation techniques and libraries. In section \ref{sec:mapping}, we put these two areas together and see why missingness and fairness are so closely entangled. We also analyse --for the running datasets and two scenarios each--  whether the examples with missing values are fairer than the rest in terms of a common fairness metric, Statistical Parity Difference (SPD). SPD sets the space of  trade-offs with performance metrics (accuracy), representing a bounding octagon, which we derive theoretically. Once this is understood conceptually,   in section \ref{sec:scenarios} we analyse a predictive model that can deal with missing values directly and see whether the bias is amplified or reduced. In section \ref{sec:imputation} we study what happens when imputation methods are introduced to allow  many other machine learning techniques to be used, and how the results compare with the models learnt by removing the missing values and with some reference classifiers (majority class and perfect classifier). 
Finally, in section \ref{sec:discussion} we analyse these results in order to make a series of recommendations about how to proceed when dealing with missing values if fairness is to be traded off against performance.

\section{Missing data}\label{sec:missing}

Missing data is a major issue in research and practice regarding real-world datasets. For instance, 
in the educational and psychological research domains, Peugh and Enders \cite{peugh2004missing} found that,  on average, 9.7\%  of the data was missing (with a maximum of 67\%). In the same area, Rombach et al. \cite{rombach2016current} estimated that this percentage ranged from 1\% to over 70\%, with a median percentage of 25\%. Another clear example of how pervasive missing data is can be found in the percentage of datasets (over 45\%) that have missing values in 
the UCI repository \cite{Dua:2019}, one of the most popular sources of datasets for machine learning and data science researchers. 

Being such a frequent phenomenon, it is not surprising that `missingness' has different {\em causes}. In the context of this paper, we need to analyse these causes if we want to properly understand the effect of missing data on fairness.
Three main patterns can be discerned in missing data \cite{lavrakas2008encyclopedia}:

\begin{itemize}

    \item \textbf{Partial completion (attrition)}. A partial completion (breakoff) is 
    produced for a single record or sequence when, after collecting a few values of a record, at a certain point in time or place within a questionnaire or a data collection process, the remaining attributes or measurements are missing. That means that attributes that are at the end and for users or cases more prone to fatigue or problems are more likely to be missing. Note that this case may have effect on rows as well, as if only a few questions (attributes) are recorded, then the whole example (row) could be removed. 
    This type of answering pattern usually occurs in longitudinal studies where a measurement is repeated after a certain period of time, or in telephone interviews and web surveys. This kind of missing data creates a dependency between  attributes that depends on their order, which should be used for imputation, and other methods for treating missing values. 
    
    \item \textbf{Missing by design}. This refers to the situation in which specific questions or attributes will not be posed to or captured for specific individuals. There are two main reasons for items to be missing by design. (1:  \textbf{contingency attributes}) Certain questions may be `non-applicable' (NA) to all individuals. In this case, the missingness mechanism is known and can be incorporated in the analyses. (2: \textbf{attribute sampling}) Specific design is used to administer different subsets of questions to different individuals (i.e., random assignment of questions to different groups of respondents). Note that in this case, all questions are applicable to all respondents, but for reasons of efficiency not all questions are posed to all respondents. Here, we also know the missingness mechanisms, but due to the random assignment of questions, this is the easiest case to treat statistically. 

    \item \textbf{Item Non-response}. No information is provided for some respondents on some variables. Some items are more likely to generate a non-response than others (e.g., items about private subjects such as income). In general, surveys, questionnaires, or interviews used to collect data suffer missing data in three main subcategories. (1: \textbf{not provided}) The information is simply not given for a certain question (e.g., an answer is not known, the value cannot be measured, a question is overlooked by accident, etc.); (2: \textbf{useless}) The information provided is unprofitable or useless (e.g., a given answer is impossible, unsuitable, unreadable, illegible, etc.); and/or (3: \textbf{lost}) usable information is lost due to a processing problem (e.g., error in data entry or data processing,  equipment fails, data corruption, etc.). The former two mechanisms originate in the data collection phase, and the latter is the result of errors in the data processing phase. 
    
\end{itemize}

\noindent While the three categories above originate from questionnaires involving people, they are general enough to include some other causes of missing values. For instance, if a speedometer stops working because it runs out of battery,  we have a partial completion, if it is only installed in some car models,  we have missing by design, and if it  is affected by the weather, we have an item non-response. 

On many occasions, for simplicity or because the traceability to the data acquisition is lost, we can only characterise some  statistical kinds of missing values. In general, a distinction is made between three types of missingness mechanisms \cite{rubin1976inference,little2014statistical}: \emph{(1)} \textit{missing completely at random} (MCAR), where missing values are independent of both unobserved and observed parameters of interest and occur entirely at random (e.g., accidentally omitting an answer on a questionnaire). In this case, missing data is independent and simple statistical treatments may be used. \emph{(2)} \textit{missing at random} (MAR), where missing values depend on observed data, but not on unobserved data (e.g., in a political opinion poll many people may refuse to answer based on demographics, then missingness depends on an observed variable (demographics), but not on the answer to the question itself). In this case, if the variable related to the missingness is available, the missingness can be handled adequately. MAR is believed to be more general and more realistic than MCAR (missing data imputation methods generally start from the MAR assumption). Finally, we have \emph{(3)} \textit{missing not at random} (MNAR), where missing values depend on unobserved data (e.g., a certain question on a questionnaire tend to be skipped deliberately by those participants with  weaker opinions).  MNAR is the most complex non-ignorable case where simple solutions no longer suffice, and an explicit model for the missingness must be included in the analysis. 

Note that all three mechanisms that generate missing data may be present \cite{enders2010applied}. 
While it is possible to test the MCAR assumption (t-test), MAR and MNAR assumptions cannot be tested due to the fact that the answer lies within the absent data \cite{molenberghs2009incomplete, enders2011missing}. 
The literature also describes different techniques to handle missing data depending upon the missingness mechanism \cite{millsap2009sage}. However, it is quite common that many practitioners apply a particular technique without analysing the missingness mechanism. Among the most common techniques, we can name the following:

\begin{itemize}
    \item Row or listwise deletion (LD) implies that whole cases (rows) with missing data are discarded from the analysis. When dealing with MCAR data  and the sample is sufficiently large, this technique has been shown to produce adequate parameter estimates. However, when the MCAR assumption is not met, listwise deletion will result in bias \cite{johnson2011toward,schafer2002missing}.
    \item Column deletion (CD) simply removes the column. This is an extreme option as it totally removes the information of an attribute. Also, it can generate bias as well, as the values of other columns may be affected by the missing values in the removed column.
    \item Labelled category (LC). An attribute (usually quantitative) can be binned or discretised, adding a special category for missing values. However, the special label representing the missing value lacks any ordinal or quantitative value. Sometimes, the existence of a missing value can be flagged with a new Boolean attribute and the original attribute is removed.  
    \item Imputation methods (IM). The missing value is replaced by a fictitious value. There are many methods to do this, such as replacing it by the mean (or the median) when the attribute is quantitative, and the mode when the attribute is qualitative, or estimating the value from the other attributes, even using predictive models. 
\end{itemize}

\noindent LD is very common in theory and IM is very common in practice. Even if the MCAR assumption is not disproven, LD is claimed to be suboptimal because of the reduction in sample size \cite{peugh2004missing,millsap2009sage}. Also, this technique for MCAR does not use the available data efficiently \cite{little2015modeling}. Therefore, methodologists have strongly advised against the use of LD  \cite{enders2010applied,little2014statistical}, judging it to be ``among the worst methods for practical applications" \cite[p. 598]{wilkinson1999statistical}. IM, on the contrary, is done by many libraries, especially as many machine learning techniques cannot deal with missing values. However, many libraries do not explicitly state what imputation method they are using, and this may vary significantly depending on the machine learning technique, library  or programming language that is used (e.g., random forest \cite{breiman2001random} handles missing values by imputation with average or mode whenever they are implemented using CART trees, but, in implementations where C4.5 trees are used instead, the missing values are not replaced and the impurity score mechanisms take them into account). 

We can analyse the missingness mechanisms in our three running datasets (Adult, Recidivism and Titanic). We use 
%
%
%
Little's MCAR global test \cite{little1988test}, a multivariate extension of the t-test approach that simultaneously evaluates mean differences on every variable in the dataset. Using the R package \textit{BaylorEdPsych}\footnote{\url{https://cran.r-project.org/web/packages/BaylorEdPsych/index.html}}, we reject the null hypothesis of MCAR
with $p$-values $< 0.001$ for the three datasets. 
MCAR is discarded for the three of them, and hence LD is inappropriate. 

Given that we do not have the complete traceability of the datasets, we can only hypothesise the causes for the missing values. For instance, 
in the Adult Census Data 
it may be due to item non-response, because of the distribution of missing values seen in Figure \ref{fig:AdultNA}. 


\section{Fairness}\label{sec:fairness}

Fairness in decision making has been recently brought to the front lines of research and the headlines of the media, but in principle, making fair decisions should be equal to making good decisions. The issue should apply both to human decisions and algorithmic decisions, but the progressive digitisation of the information used for decision making in almost every domain has facilitated the detection and assessment of systematic discrimination against some particular groups.  The use of data processing pipelines is then a blessing and not a curse for fairness, as it makes it possible to detect (and treat) discrimination in a wider range of situations than when humans make decisions solely based on their intuition. As we will see, metrics for fairness have led to computerised techniques to improve fairness (mitigation techniques). However, it is still very important to know the {\em causes} of discrimination, to avoid oversimplifying the problem. Also, in this paper we want to map these causes with those of missing values seen in the previous section.

What are the causes that introduce unfairness in the decision process? We can identify common distortions originally arising in the data, but also those in the algorithms or the humans involved in decision making that can perpetuate or even amplify some unfair behaviours. From the literature \cite{heckman1979sample,millsap1993methodology,donaldson2002understanding,nickerson1998confirmation,devine2002regulation,chouldechova2018frontiers,li2014examining,wijnhoven2014external,kunz2016effect}, we can classify them in six main groups:

\begin{itemize}

\item \textbf{Sample or selection bias}. It occurs when the (sample of) data 
is not representative of the target population about which conclusions are to be drawn. This happens due to the sample being collected in such a way that some members of the intended population are less likely to be included than others.  
A classic example of a biased sample happens in politics, such as the famous 1936 opinion polling for the U.S. presidential elections carried out by the American Literary Digest magazine, which over-represented rich individuals and predicted the wrong outcome 
\cite{squire19881936}. 
If some groups are known to be under-represented and the degree of under-representation can be quantified, then sample weights can correct the bias.

\item \textbf{Measurement bias (systematic errors)}. Systematic value distortion happens when the device used to observe or measure favours a particular result (e.g., a scale which is not properly calibrated might consistently understate weight producing unreliable results). This kind of bias is different from random or non-systematic measurement errors whose effects average out over a set of measurements. On the contrary, systematic errors cannot be avoided simply by collecting more data, but having multiple measuring devices (or observers of instruments), and specialists to compare the output of these devices. 

\item \textbf{Self-reporting bias (survey bias)}. This has to do with non-response, incomplete and inconsistent responses to surveys, questionnaires, or interviews used to collect data. The main reason is the existence of questions concerning private or sensitive topics (e.g., drug use, sex, race, income, violence, etc.). Therefore, self-reporting data can be affected by two types of external bias: \emph{(1)} social desirability or approval (e.g., when determining drug usage among a sample of individuals, the actual average value is usually underestimated); and \emph{(2)} recall error  (e.g.,  participants can erroneously provide responses to a self-report of dietary intake depending on her ability to recall past events). 

\item \textbf{Confirmation bias (observer bias)}. It is described as placing emphasis on one hypothesis because it involves favouring information that does not contradict the researcher’s desire to find a statistically significant result (or its previously existing beliefs). This is a type of cognitive bias in which a decision is made according to the subject’s preconceptions, beliefs, or preferences, but can also emerge owing to overconfidence (with contradictory results/evidence evidence being overlooked). 
Peter O. Gray \cite{psychology} provides an example of how confirmation bias may affect a doctor's diagnosis: \textit{"When the doctor has jumped to a particular hypothesis as to what disease a patient has may then ask questions and look for evidence that tends to confirm that diagnosis while overlooking evidence that would tend to disconfirm it"}. 

\item \textbf{Prejudice bias (human bias)}. A different situation is when the training data that we have on hand already includes (human) biases containing implicit racial, gender, or ideological prejudices. Unlike the previous categories, which mostly affect the predictive attributes (model inputs), this kind of bias is concerned with variables that are used as dependent variables (model outputs). Therefore, systems designed to reduce prediction error will naturally replicate any bias already present in the labelled data. We find examples of this in the re-offence risk assessment tool COMPAS deployed in federal US criminal justice systems to inform bail and parole decisions  that demonstrated biased against black people \cite{angwin2016machine}. Another one is the Amazon’s AI hiring and recruitment system that showed a clear bias against women \cite{AmazonBias}, having been trained from CVs submitted to the company over a 10-year period.

\item \textbf{Algorithm bias}. In this case the algorithm creates or amplifies the bias over the training data. 
For instance, different populations in the data may have different feature distributions (also having different relationships to the class label). As a result, if we train a group-blind classifier to minimise overall error, as it cannot usually fit both populations optimally, it will fit the majority population. It may be plausible that the best classifier is one that always picks the majority class, or ignores the values for some minority group attributes leading, thus, to (potentially) higher distribution of errors in the minority population. 

\end{itemize}

\noindent While the range of causes in general can be enumerated, the precise notion of fairness and what causes it in a particular case is much more cumbersome. It is nonetheless very helpful to become more precise with definitions or metrics of fairness, based on the notion of protected attribute and parity. Let us start with the definition of a decision problem, which is tantamount to a supervised task in machine learning. We will
focus on classification problems, since it is the most common case in the fairness literature and the metrics are simpler.

Let us define a set of attributes $X$, where the subset $S$ denotes the protected attributes, which are assumed to be categorical. For each protected attribute $S_i$, we have a set of values $V_i$ (e.g., $\{\mbox{\tt male}, \mbox{\tt female}\}$). Sometimes one of the groups that can be created by setting the value of one or more protected attributes is considered the `privileged' group (e.g., males). Usually, fairness metrics are based on determining whether decisions are different between groups or just between the privileged groups and the rest. Now consider a label or class attribute $Y$, which can take values in $C$ (e.g., $\{\mbox{\tt guilty}, \mbox{\tt non-guilty}\}$). For the analysis of fairness we usually consider one of the classes as the ``favourable" (or positive) outcome, denoted by $c^+$. Similarly, $c^-$ denotes the unfavourable class (or classes, in a multiclass problem, all the other classes together). For instance, in this case, {\tt non-guilty} is the favourable outcome. 
An unlabelled example $x$ is a tuple choosing values in $V_i$ for each $X_i$, possibly including the extra value $\missing$, representing a missing value. A labelled example or instance $\langle x, y \rangle$ is formed from an unlabelled example with the class value $y$ taken from $C$. A decision problem is just defined as mapping $x$ to $\hat{y}$, such that $\hat{y}$ is correct with respect to some ground truth $y$. Let us denote with $M$ a mechanism or model (human or algorithmic) that tries to solve the decision problem. 
Datasets, samples and populations are defined over sets or multisets of examples, and examples are drawn or sampled from them, using the notation $\sim D$. Given a population or dataset $D$ we denote by $D_{X_i = a}$ the selection of instances such that $X_i = a$, with $a \in V_i$. 
This also applies to labelled datasets, so, $D_{y=c^+}$ is just the number of positive examples in $D$, usually shortened as $\Pos(D)$ (we do similarly for the negative examples $\Neg(D)$). By comparing the true positives and negatives with predicted positives and negatives of a model (as a dataset with estimated labels $\hat{D}$), we can define TP, TN, FP and FN, as usual. %
Finally, given pairs $\langle x, y \rangle$ from a dataset or distribution $D$, the probability of favourable outcome $p(y=c^+)$ when $x \sim D$ is simply denoted by $p^+(D)$. 

Some fairness metrics choose an indicator that does not depend on the confusion matrix, but just on overall probabilities, for instance, predicting the favourable class or predicting the unfavourable class, so they can be applied to datasets and the predictions of a model. Other metrics are defined in terms of comparing predicted and true labels, e.g., comparing the true positive rate (TPR) and the false positive rate (FPR), so they can only be applied to models comparing them to the ground truth (or a test dataset). In the end, several combinations of cells in the confusion matrix lead to dozens of fairness metrics. 

A very common metric is the Statistical Parity Difference (SPD), which --for the favourable class-- can be defined for an attribute $X_i$ with privileged value $a$ as follows:
%
\[ \SPD^+_i(D) = p^+(D_{X_i = a}) - p^+(D_{X_i \neq a}).\]
\noindent For instance, if the attribute is {\tt Race} and the values are {\tt caucasian}, {\tt black} and {\tt asian}, if we consider {\tt caucasian} as the privileged group, SPD would be the probability of favourable outcomes for
Caucasians minus the probability of favourable outcomes for non-Caucasians. 
A value of 0 implies both groups have equal benefit, a value less than 0 implies higher benefit for the 
unprivileged group, 
and a value greater than 0 implies higher benefit for the 
privileged group. Note that the sign of SPD will change if for a binary dataset we swap the favourable and unfavourable class:
\begin{eqnarray*}
\SPD^-_i(D) & = &  p^-(D_{X_i = a}) - p^-(D_{X_i \neq a}) = 1 -  p^+(D_{X_i = a}) - (1 - p^+(D_{X_i \neq a})) 
= -\SPD^+_i(D).
\end{eqnarray*}
The same happens if we swap the privileged groups. This is an interesting property, as the choice of the favourable class  and especially the privileged group is sometimes arbitrary. For instance, in the adult dataset, if a model is used to grant a subsidy, the favourable class is earning $<$\$50K. The important thing is that a value closer to zero is fairer. 

Other popular metrics that are applicable to both datasets and models are Disparate Impact (DI), which calculates a ratio instead of the difference as in SPD. Some popular metrics that are only applicable to models are the Equal Opportunity Difference (EOD), which is the difference in TPR between the groups, or the Average Odds Difference (OddsDif), which also considers the FPR. 
There is usually some confusion about what fairness metric to look at, something that is not very different from the choice of performance metric in classification \cite{ferri2009experimental,hernandez2012unified}. Our experience, at least for the datasets we have selected for this study, is that  
we reach similar conclusions independently of the metric we choose for the study. As a result, in what follows, and for the sake of exposition and simplicity, we will use SPD as a representative fairness metric.

The introduction of formal metrics have fuelled the development of new techniques for discrimination mitigation. There are three main families: those who are applied to the data prior to learning the model, those that modify the learning algorithm, and whose that modify or reframe the predictions (these are known as pre-processing, in-processing and post-processing respectively in \cite{aif360-oct-2018}). Typically, maximising one fairness metric has a significant effect on the performance metrics, such as prediction error, so it is quite common to find trade-offs between fairness metrics and performance metrics. A possible way of doing this is through a Pareto optimisation, where different techniques can be compared to see whether they improve the Pareto front. We will use this approach in the rest of the paper.

Finally, both metrics and mitigation techniques are usually integrated into libraries. These are the most representative ones, in our opinion: 

\begin{itemize}
\item AIF360 toolkit \cite{aif360-oct-2018}: is an open-source library to help detect and remove bias in machine learning models. The AI Fairness 360 Python package includes a comprehensive set of metrics for datasets and models to test for biases, explanations for these metrics, and algorithms to mitigate bias in datasets and models.
\item Aequitas \cite{2018aequitas}: is an open-source bias audit toolkit for data scientists, machine learning researchers, and policymakers to audit machine learning models for discrimination and bias, and to make informed and equitable decisions around developing and deploying predictive risk-assessment tools.
\item Themis-ML \cite{bantilan2018themis}: Is a Python library that implements some fairness metrics (Mean difference and Normalized mean difference) and fairness-aware methods such as Relabelling (pre-processing), Additive Counterfactually Fair Estimator (in-processing) and Reject Option Classification (post-processing), providing a handy interface to use them. The library is open for being extended with new metrics, mitigation algorithms and test-bed datasets.
\item Fairness-Comparison library \cite{friedler2019comparative}: Presented as a benchmark for the comparison of the different bias mitigation algorithms, this Python package makes available to the user the set of metrics and fairness-aware methods used for the study. It also allows its extension by adding new algorithms and datasets.
\end{itemize}
Some of these libraries come (or are illustrated) with datasets that are known to have or lead to fairness issues. 
In our case, we chose those datasets from the literature that have missing values originally or we introduced new datasets, such as Titanic, where we think fairness issues are relevant (although somewhat in the opposite direction as usual). The details of these datasets were given in the introduction.

\section{Mapping missingness and unfairness}\label{sec:mapping}

The new question we ask in this paper is the effect of missing values for fairness. As we will see, this is a complex relation for which we need to map the causes of missing values to the causes of unfair treatment, as illustrated in Figure \ref{fig:mapping}. 
Looking at the lefthand side of the figure, we see that missing values might be the consequence of innumerable factors, from basic errors while processing and acquiring data to intentional 
action from human agents. When fairness is taken into consideration, one must realise that the missing data might not be evenly distributed between different groups, which in turn might lead unwanted effects on the fairness of the data and models created, depending on how the missing data are handled. As an example, on the side of errors, an important factor is that different groups might have different semantic constructs to answer the same query, leading to different interpretations and omissions \cite{Milfont2010}, e.g., leading to more missing values on a sensitive group than the other. On the other extreme, people might intentionally omit information as a natural cooping mechanism when there is a belief that a truthful and complete answer might lead to a discriminatory and unfair decision \cite{Schwarz1998}.

If we go one by one from the causes of missingness to unfairness or vice versa, we see that many combinations are related. While {\em attrition}, {\em attribute sampling} (if random) and {\em lost information} may be less associated with some other causes of unfairness, some others, such as {\em contingency questions}, {\em not provided} and {\em useless answers} may be  strongly related to fairness. It is not difficult to see that there are common underlying causes for these missingness situations and {\em self-reporting (survey) bias}, {\em confirmation (observer) bias} and {\em prejudice (human) bias}. What we do not know is whether some of the conscious or unconscious actions done by the actors in the process may have a compensatory result. For instance, are women who do not declare their number of children treated in a more or less fair way than those who declare (or lie) on their number of children? In the end, filling a questionnaire or simply acting when a person knows that is being observed with the purpose of building a model from their behaviour triggers many mechanisms in which people can conceal their real information or behaviour, in order to be classified in their desired group. In other words, some types of missing values can be used in an adversarial way by the person being modelled, trying to be classified into the favourable outcome. All this happens with university admissions, credit scoring, job applications, dating apps, etc.

\begin{figure}[!htb]
        \center{\includegraphics[width=0.5\textwidth]
        {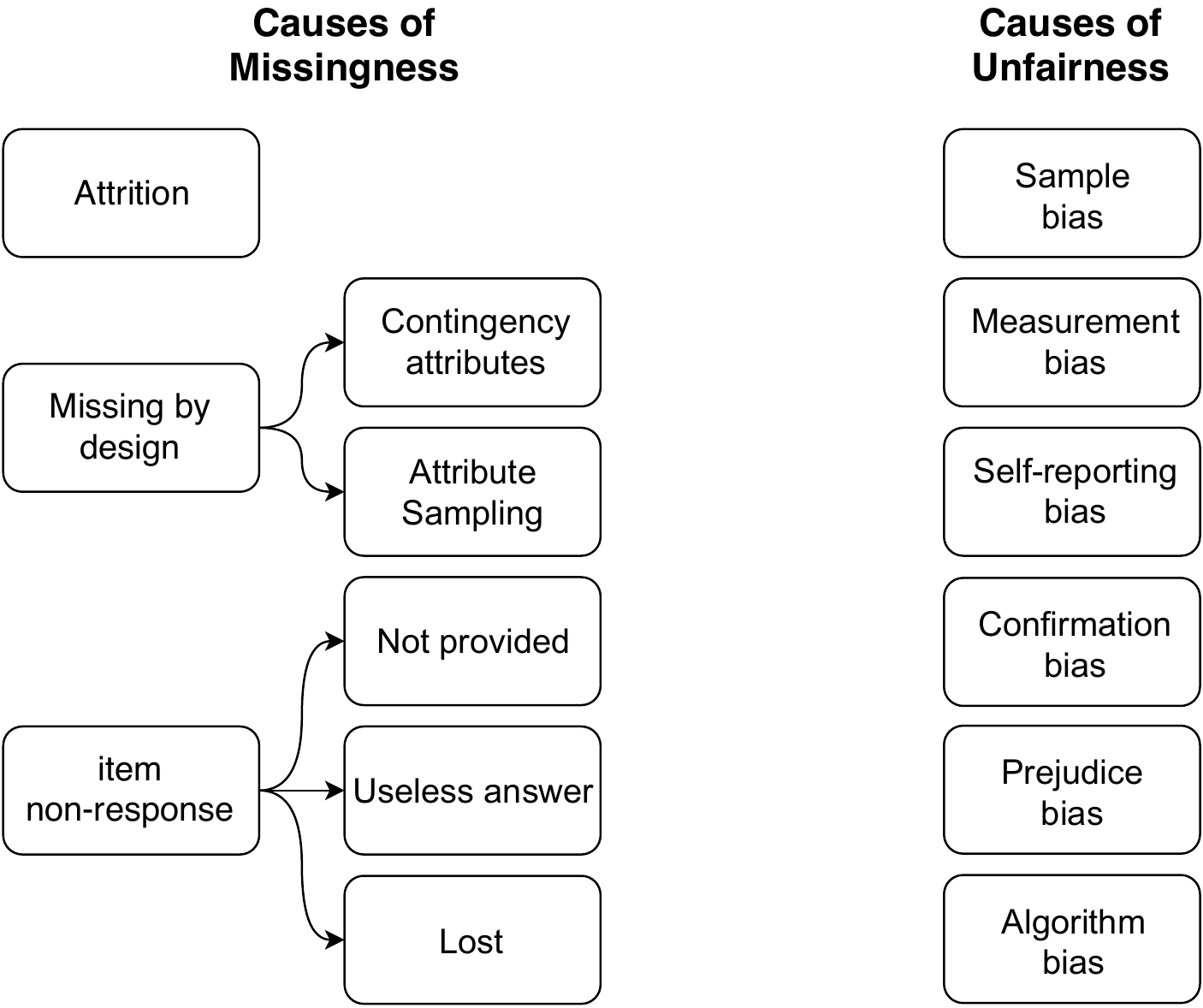}}
        \caption{\label{fig:mapping} Mapping between causes of missingness and unfairness.
        }
\end{figure}

We may get more certainty about the relation between missingness and unfairness by analysing some real data. Let us look at our three running datasets, Adult, Recidivism and Titanic, each of them with two protected attributes, which allows us to perform six different evaluations\footnote{For the sake of reproducibility, all the code and data can be found in \url{https://github.com/nandomp/missingFairness}}. First, we do a simple correlation analysis, as shown in Figure \ref{fig:corrs}. What we see is that having a missing value or not shows small correlations with the protected attributes and the class, although this is affected by the proportion of missing values per attribute being relatively small. The purpose of this correlation matrix is actually to confirm that  strong correlations are not found, and the relations, if they exist, are usually more subtle.  

 \begin{figure}[!htb]
          \centering
          {\includegraphics[width=0.3205\textwidth]{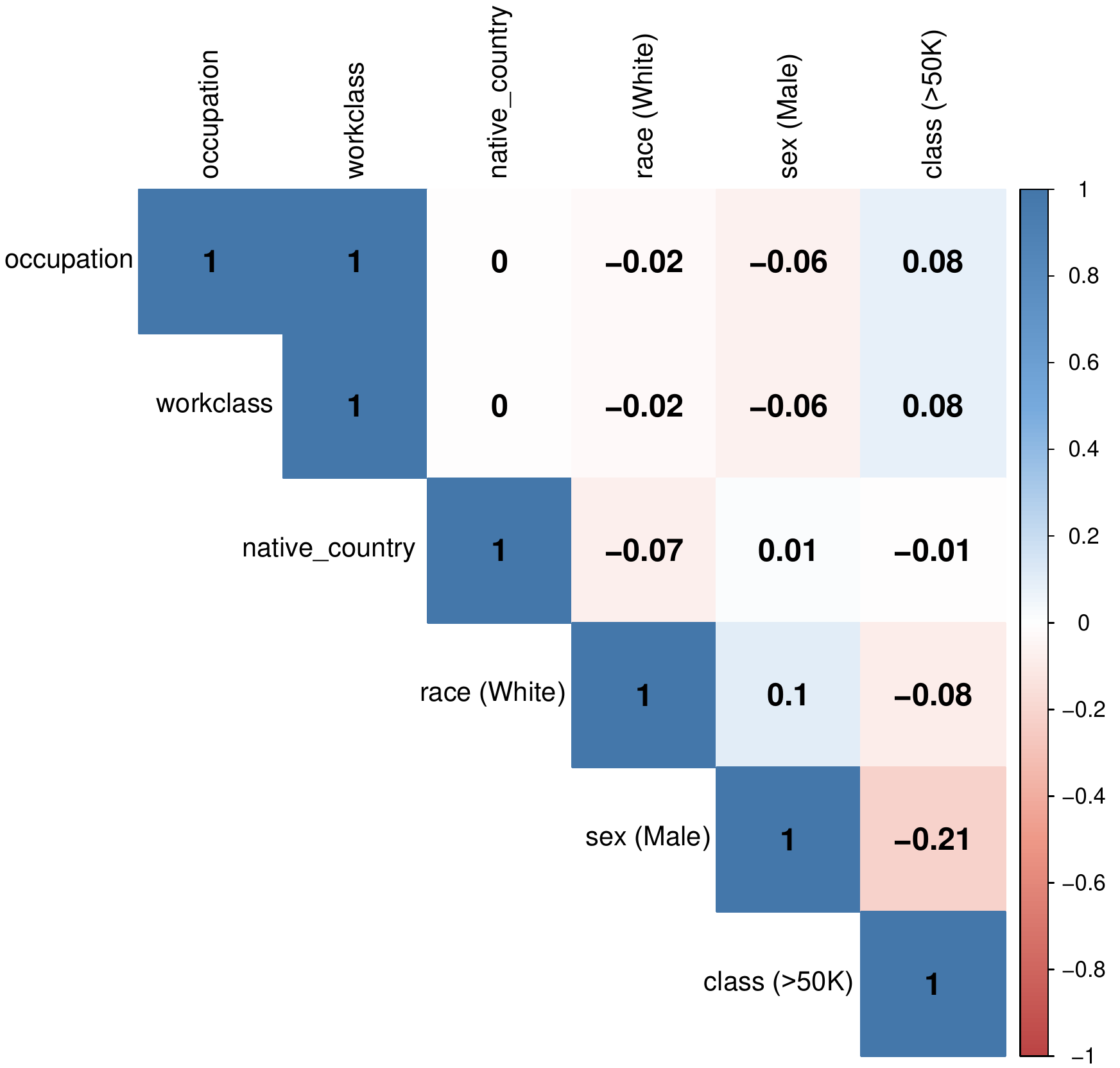}\label{fig:corrAdult}}
          \hfill
          {\includegraphics[width=0.36\textwidth]{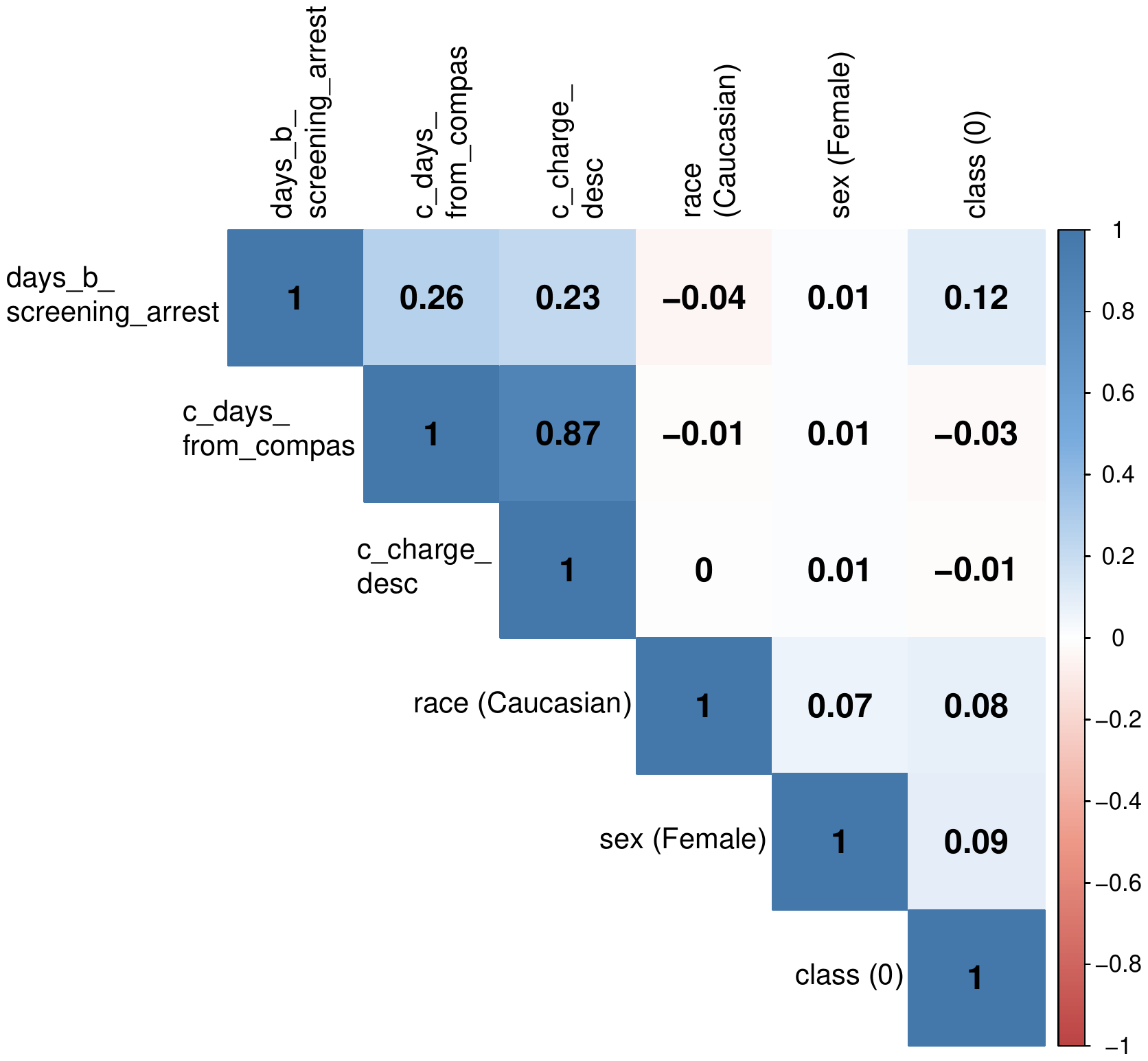}\label{fig:corrRecid}}
          \hfill
          {\includegraphics[width=0.303\textwidth]{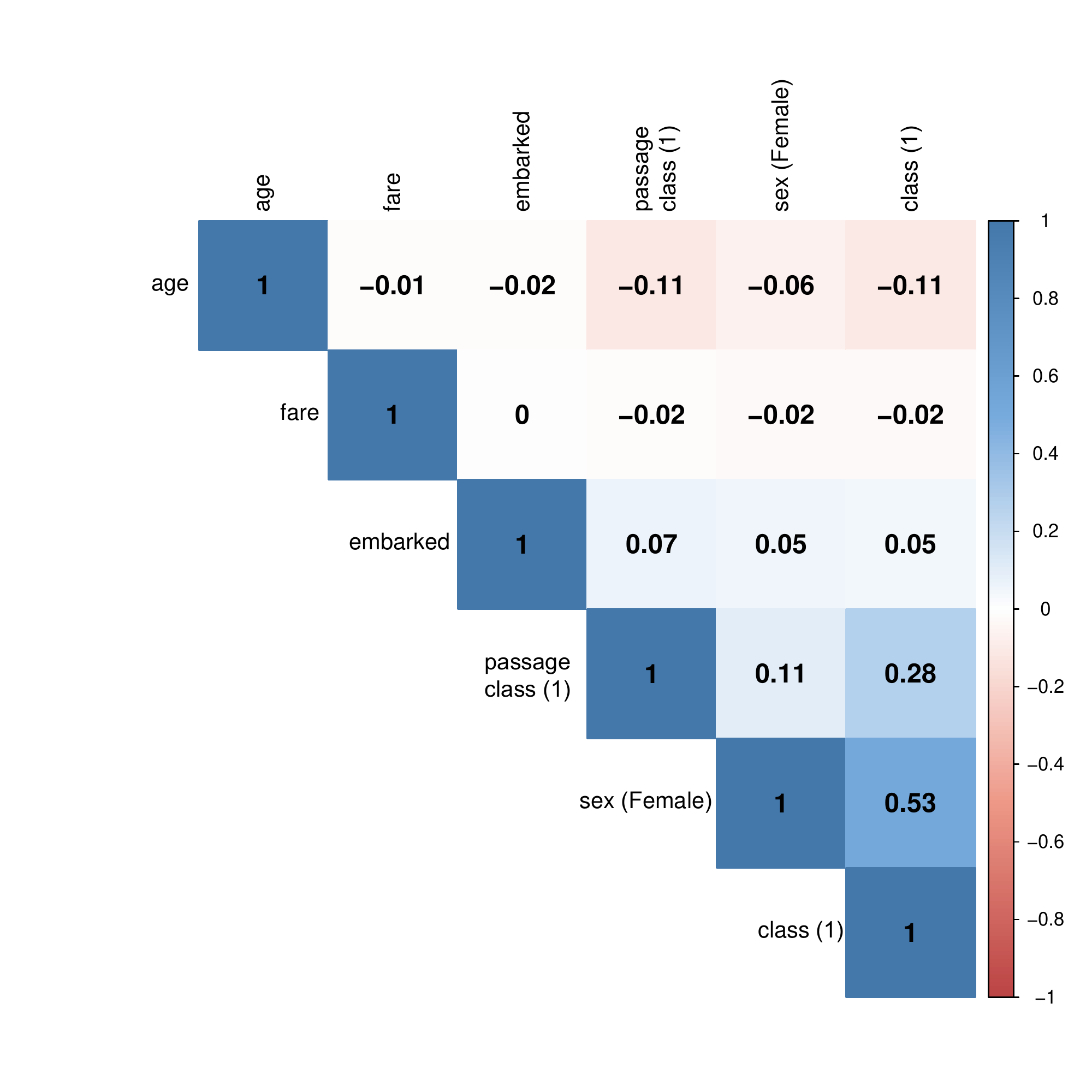}\label{fig:corrTitanic}}
          \caption{\label{fig:corrs} Correlation matrices for the Adult (left), Recidivism (centre) and Titanic (right) datasets. For the whole dataset We analyse  Pearson's  correlations between a discretised version of the attributes that may contain missing values (1 if it is missing and 0 otherwise), the protected attributes (1 if it is privileged and 0 otherwise) and the positive (favourable) class. 
          }
\end{figure}

This is what we see when we analyse a fairness metric. For each of these six cases (3 datasets $\times$ 2 privileged groups), Table \ref{tab:PRE} shows the fairness metric (SPD) for the whole data set (``all rows''), for the subset of examples with missing values (``with \missing'') and for the rest (``w/o \missing''). If we look at the value of SPD for all rows for Adult, we see that this is positive. 
This means higher benefit for the privileged groups, whites and males. 
Nevertheless, it is important to note that the favourable class for Adult is earning more than \$50K. If a model is used to determine who is entitled to a subsidy, then having the label $>$\$50K would not the favourable outcome. 
This suggests that we should take the sign in a more relative way, paying attention to the absolute magnitudes. 
 A similar result is seen for Recidivism, where the positive values for SPD indicate that the privileged groups (Caucasian and female) are associated with no recidivism more often than the other groups. Finally,  
 Titanic, where the positive (and high) values for SPD indicate that the privileged groups (first class passengers and females) fared very favourably (their survival rate was much higher than the complementary groups). Here, the bias is much stronger, as this followed an explicit protocol favouring females in ship evacuation at the time, and many other conditions favouring first class passengers very clearly.

\newcolumntype{L}[1]{>{\raggedright\arraybackslash}p{#1}}
\newcolumntype{C}[1]{>{\centering\arraybackslash}p{#1}}
\newcolumntype{R}[1]{>{\raggedleft\arraybackslash}p{#1}}

\setlength{\tabcolsep}{1.6pt}
\renewcommand{\arraystretch}{1.2}

\begin{table}[ht]
\centering
\resizebox{0.95\textwidth}{!}{%
\begin{tabular}{lcC{2.5cm}C{1.85cm}C{2cm}C{1.75cm}C{1.75cm}C{2cm}C{2.75cm}C{2cm}C{2cm}C{2cm}}\toprule
\textbf{dataset} &  \textbf{\#rows} & \textbf{\#rows with \missing}  & \textbf{\#cols with \missing} & \textbf{protected attribute}& \textbf{privileged values}& \textbf{unprivlgd. values} & \textbf{$c^+$}& \textbf{majority}& \textbf{SPD (all rows)} & \textbf{SPD (rows with \missing)} & \textbf{SPD (rows w/o \missing)} \\ 
  \hline
\multirow{ 2}{*}{Adult} &  \multirow{ 2}{*}{48842} &  \multirow{ 2}{*}{3620 (7.4\%)} &     \multirow{ 2}{*}{3} & Race& white & $\neq$ white & $>$\$50K&   $\leq$ \$50K (76\%) & 0.1014 & \textbf{0.0361} & 0.1040 \\ 
                        &   &   &      & Sex & male &female & >\$50K &$\leq$ \$50K (76\%)& 0.1945 & \textbf{0.1117} & 0.1989 \\ \hline

\multirow{ 2}{*}{Recidivism} & \multirow{ 2}{*}{7214} &   \multirow{ 2}{*}{314 (4.4\%)} &  \multirow{ 2}{*}{3} & Race & caucasian & $\neq$ caucasian& 0 (no crime)& 0 (no crime) (55\%) & 0.0864 & \textbf{0.0716} & 0.0920 \\ 
 &  &    &   & Sex & female & male & 0 (no crime)& 0 (no crime) (55\%) & 0.1161 & \textbf{0.0243} & 0.1186 \\ \hline

\multirow{ 2}{*}{Titanic} & \multirow{ 2}{*}{1309} &  \multirow{ 2}{*}{26 (20.3\%)} &  \multirow{ 2}{*}{3} & Class & 1&2-3& 1 (survived) & 0 (died) (62\%) & 0.3149 & \textbf{0.2722} & 0.3115 \\ 
 &   &    &    & Sex & female & male & 1 (survived) & 0 (died) (62\%) & 0.5365 & \textbf{0.4727} & 0.5458 \\ 

\toprule
\end{tabular}
}
\caption{
Dataset description and fairness metrics (SPD) for different subsets of data. We indicate the number of rows and columns with missing values (\missing), the protected attribute, its privileged and unprivileged values,  the favourable class $c^+$ and the majority class (with the proportion for all rows). Bold figures represent the fairest subset.
}
\label{tab:PRE}
\end{table}

The really striking result appears when we look at the metrics for the subsets of instances  with missing values (``SDP (rows with \missing)"). The six cases consistently show that these subsets  are {\em fairer} (their values closer to zero) than the subset of instances containing only clean rows (without missing values). How do we interpret this finding? It is difficult to explain without delving into the particular characteristics of each dataset. For instance, for Adult, this means that for those individuals with missing values in occupation, workclass or native.country, there is a smaller difference in the probability of a favourable outcome ($>$50K) between privileged and unprivileged groups than for the other rows. Note that this observation is about the data, we are not talking about algorithmic bias yet.

Algorithmic bias will be affected by the metric that is used for performance. If the metric is a proper scoring rule, the optimal value is attained with the perfect model, i.e., a model that is 100\% correct. But a perfect classifier would necessarily have the fairness metrics of the test dataset. This leads us to the following observation: the SPD of the test dataset determines the space of possible classifiers, i.e., a region that is bounded by the best trade-off between the performance metric and SPD. We can precisely characterise this region. For instance, if we choose Accuracy as performance metric, the following proposition shows this space of possible classifiers in terms of the balance of SPD and Accuracy. 

\begin{proposition}\label{prop:octagon}
Assuming that the majority class is the favourable class, the set of classifiers with respect to Accuracy and SPD is bounded by the octagon defined by the following eight points $(\Acc, \SPD)$:
\[\left(1, \:\:\:\:\:\:\:\:\:\:\:\:\:\:\:\:\:\:  p^+(D_{X_i = a}) - p^+(D_{X_i \neq a}) \right)\]
\[  \left(  1 - \frac{\Pos({D}_{X_i = a})}{|D|}, \:\:\:\:\:\:   - \frac{\Pos({D}_{X_i \neq a})}{|D_{X_i \neq a}|}
\right) \]
\[ \left( \frac{\Neg({D}_{X_i = a}) + \Pos({D}_{X_i \neq a})}{|{D}|} , \:\:\:\:\:\:    - 1 \right) \]
\[\left( \frac{\Neg({D}_{X_i = a})}{|{D}|} , \:\:\:\:\:\: - 1 + \frac{  \Pos({D}_{X_i \neq a})}{|D_{X_i \neq a}|} \right) \]
\[\left( 0 , \:\:\:\:\:\:\:\:\:\:\:\:\:\:\:\:\:\:   p^+(D_{X_i \neq a}) - p^+(D_{X_i = a})\right)\]
\[\left(  \frac{\Pos({D}_{X_i = a})}{|D|}, \:\:\:\:\:\:\:\:\:\:\:\:\:\:\:\:\:\:     \frac{\Pos({D}_{X_i \neq a})}{|D_{X_i \neq a}|}\right)
\]
\[\left( \frac{\Pos({D}_{X_i = a}) + \Neg({D}_{X_i \neq a})}{|{D}|} , \:\:\:\:\:\:\:\:\:\:\:   1 \right)\]
\[\left( 1 - \frac{\Neg({D}_{X_i = a})}{|{D}|} , \:\:\:  
1 - \frac{\Pos({D}_{X_i \neq a})}{|D_{X_i \neq a}|} \right)
 \]
\end{proposition}

The proof and a graphical representation of this octagon can be found in Appendix \ref{sec:proof}.

If we observe Table \ref{tab:PRE}, we can see that the SPD-accuracy space of the complete case can reach better values than the space configured by the dataset without missing values. This is even more extreme for the subset only including the missing values. A perfect classifier for these datasets could have almost perfect SPD for Adult and Race, and Recidivism and Sex. The perfect classifier for the subset of examples without missing values would be very unfair. So there are two problems with choosing this mutilated dataset. First, the perfect `target' that this dataset determines is wrong, because any model has to be evaluated with all the examples, not only a sample that is not chosen at random. Second, it will also lead to an unfair model. In the following sections we will explore the location of classifiers in this space and see how far we can approach the boundary of the possible space.

Independently of the explanation of the remarkable finding that the subset with missing values is fairer than the rest --for each and every of the six cases in Table \ref{tab:PRE}--, the main question is: if rows with missing values are fairer, why is everybody getting rid of them when learning predictive models? 
As we saw, some of the libraries seen in the previous section apply either $LD$ (e.g., AIF360, Fairness-Comparison) or $CD$ (Themis-ML), or they assume that the dataset is clean of missings (Aequitas). 
These libraries do this because many machine learning algorithms (and all fairness mitigation methods) cannot handle missing values. Apparently, $LD$ seems the easiest way of getting rid of them. However, do we have better alternatives? As we saw, imputation methods could replace the missing values and keep these rows, which seem to have less biased information, as we have seen in the previous six cases. Before exploring the results with $LD$ and a common imputation method, we need to explore how these missing values affect machine learning models and the fairness metrics of the trained models, in comparison with the original fairness metrics. In other words, do rows with missing values contribute to bias amplification more or less than the other rows? That is what we explore in the following section.

\section{Regaining the missing values for fairness}\label{sec:scenarios}

Since missing values are ugly and uncomfortable, they are eliminated in one way or another before learning can take place. Actually, most machine learning methods (or at least most of their off-the-shelf implementations) cannot handle missing values directly. There are a few exceptions, and analysing results with a method that does consider the missing values can shed light on the effect of missing values on fairness when learning predictive models. One such an exception is decision trees. During training, missing values are ignored when building the split at each node of the tree. During the application of the model, if a condition cannot be resolved because of a missing values, the example can go through some of the children randomly or can go through all of them and aggregate the probabilities of each class.  Another important reason why we choose decision trees for this first analysis is that they can become understandable (if of moderate size) and ultimately inspectable, which allows us to see what happens with the missing values and where unfairness is created. Also, the importance of each attribute can be derived easily from the tree.

In particular, we are using the classical Classification and Regression Trees (CART) \cite{breiman1984classification} in the implementation provided by the \emph{rpart} package \cite{therneau2015package}. This implementation treats missing values using \emph{surrogate splitting}, which allows the use of the values of other input variables to perform a split for observations with a missing value for the best (original) split (see section 5 in \cite{therneau1997introduction} for further information). Because the three datasets we are using are not very imbalanced (at most 76\% for Adult), and because Accuracy is the most common metric in many studies of fairness, we will stick to this performance metric. 

We separated a \% of the data for test from the original dataset, disregarding the existence of missing values. Consequently, this test dataset has a mixture of rows with and without missing values approximately equal as the whole dataset. Then, for training the decision tree, we used four different training sets. The ``all rows'' case used all the rows not used for test. The ``with \missing" case used the subset of these that have missing values. The ``without \missing'' case used the subset of the ``all rows'' training set whose rows do not have missing values. Finally, for comparison, we made a small sample of the latter of the same size of the training set ``with \missing''. We used 100 repetitions of the training/test split, 
where the fairness metrics (and the accuracies) are calculated with the test set labelled with the decision tree, and then averaged for the 100 repetitions. 
Table \ref{tab:IN} shows the results of CART decision trees for the test set and the six cases we are considering. We first compare the fairness results of the model over the data with all rows with the original fairness results of the dataset as we saw in Table \ref{tab:PRE}. We see that for Adult the bias is reduced (represented by \reduces) for Adult but worse for Recidivism and Titanic, for which bias is amplified (represented by \amplifies). 
 We think that this difference for Adult being biased in favour of the unprivileged class (as for the negative values of the SPD metric) may be do to a questionable choice of the favourable class here; if a model for predicting income is used, e.g., for granting a subsidy, being classified as earning $>$\$50k is unfavourable.  
When we look at other subsets, again comparing them with the data fairness in Table \ref{tab:PRE}, we see a similar picture for the subset without missing values, but almost the opposite situation with the subset with missing values. Again, the rows with missing values are having a very different behaviour.
Still, if we analyse which subset is best to get the least bias model, we see that 
learning from the rows with missing values is better for fairness. We added the sample at the rightmost columns (``sample w/o \missing'') to show whether this was caused by CART having shorter (and less biased) trees when the sample is small. We see this effect only partially. 
The size of the training sample does not explain many of the gains for the subset with missing values.

\begin{table}[ht]
\centering
\resizebox{1\textwidth}{!}{%
\begin{tabular}{lC{1.5cm}C{2.5cm}C{2.75cm}C{2.5cm}C{2.95cm}C{2.5cm}C{2.75cm}C{2.5cm}C{2.5cm}}\toprule

\textbf{dataset} & \textbf{protected attribute} & \textbf{Acc (all rows)}  & \textbf{SPD (all rows)} &  \textbf{Acc (with \missing)}  & \textbf{SPD (with \missing)}  & \textbf{Acc (w/o \missing)}  & \textbf{SPD (w/o \missing)} & \textbf{Acc (sample w/o \missing)}  & \textbf{SPD (sample w/o \missing)} \\ 
  \hline
\multirow{ 2}{*}{Adult}  
& Race &
0.8504 $\pm$ 0.0032 & \reduces0.0876 $\pm$ 0.0096 &
0.8264 $\pm$ 0.0075 & \amplifies\textbf{0.0768 $\pm$ 0.0179*} &
0.8502 $\pm$ 0.0031 & \reduces0.0881 $\pm$ 0.0102 &
0.8284 $\pm$ 0.0063 & -0.0888 $\pm$ 0.0163 \\ 
& Sex & 
0.8504 $\pm$ 0.0032 & \reduces0.1868 $\pm$ 0.0072 &
0.8264 $\pm$ 0.0075 & \amplifies\textbf{0.1643 $\pm$ 0.0220*} &
0.8502 $\pm$ 0.0031 & \reduces0.1885 $\pm$ 0.0072 & 
0.8284 $\pm$ 0.0063 & -0.1959 $\pm$ 0.0182\\ \hline
  
\multirow{ 2}{*}{Recidivism} 
& Race & 
0.6227 $\pm$ 0.0089 & \amplifies0.0982 $\pm$ 0.0289 & 
0.5549 $\pm$ 0.0165 & \reduces\textbf{0.0347 $\pm$ 0.0315*} & 
0.6214 $\pm$ 0.0114 & \amplifies0.1013 $\pm$ 0.0262 & 
0.5793 $\pm$ 0.0193 & 0.0509 $\pm$ 0.0404\\ 
& Sex & 
0.6227 $\pm$ 0.0089 & \amplifies0.1384 $\pm$ 0.0310 & 
0.5549 $\pm$ 0.0165 & \reduces\textbf{-0.0093 $\pm$ 0.0419*} & 
0.6214 $\pm$ 0.0114 & \amplifies0.1351 $\pm$ 0.0328 & 
0.5793 $\pm$ 0.0193 & 0.0383 $\pm$ 0.0403 \\ \hline
  
\multirow{ 2}{*}{Titanic}
& Class & 
0.7819 $\pm$ 0.0210 & \amplifies0.3507 $\pm$ 0.0701 & 
0.7113 $\pm$ 0.0296 & \amplifies\textbf{0.2875 $\pm$ 0.1213*} & 
0.7724 $\pm$ 0.0227 & \amplifies0.3641 $\pm$ 0.0766 & 
0.7451 $\pm$ 0.0282 & 0.2886 $\pm$ 0.1556 \\
  
& Sex & 
0.7819 $\pm$ 0.0210 & \amplifies0.6692 $\pm$ 0.0572 & 
0.7113 $\pm$ 0.0296 & \reduces\textbf{0.4418 $\pm$ 0.1308*} & 
0.7724 $\pm$ 0.0227 & \amplifies0.6471 $\pm$ 0.0670 & 
0.7451 $\pm$ 0.0282 & 0.6827 $\pm$ 0.1519\\ 
\toprule
\end{tabular}
}
\caption{Fairness metric (SPD) and accuracy (Acc) averaged for 100 repetitions using CART as predictive model. We show results for different subsets of data: all rows, only the rows with missing values (\missing), only the rows without the missing values and a sample of the latter of the same size. The symbols \amplifies{} and \reduces{} represent whether the bias has been amplified or reduced respectively in comparison to the corresponding columns in Table \ref{tab:PRE}. Bold figures represent the fairest result (closest to 0). The star symbol denotes statistical significance in a multiple pairwise-comparison between the means of the columns \emph{SPD (with \missing)}, \emph{SPD (w/o \missing)} and \emph{SPD (sample w/o \missing)}. 
}
\label{tab:IN}
\end{table}

Of course, this analysis disregards the performance of the model. We know that we can easily obtain a perfectly fair model by using trivial models such as the one that always predicts the majority class (with a baseline accuracy). 
Figure \ref{fig:IN} shows a bidimensional representation with the performance metric (accuracy) on the \yaxis and the fairness metric (SPD) on the \xaxis, where this majority class model has also been introduced. By looking at the trade-off between accuracy and SPD, the picture gets more complex but still more interesting. While the model only learning from the rows with missing values shows poor accuracy (because the sample is small), and it is usually out of a possible Pareto of accuracy and SPD, the result including all rows is usually the best in this Pareto. It always has the highest accuracy, and it is fairer than the subset without missing values for four of the six cases (Adult with race and sex, Recidivism with race, and Titanic with class). Although this is not conclusive evidence that including the rows with missing attributes is good, it is still indicative that ignoring them altogether may not be a good choice. We need to explore further. 

\begin{figure}[ht]
	\centering
    \includegraphics[width=1\columnwidth]{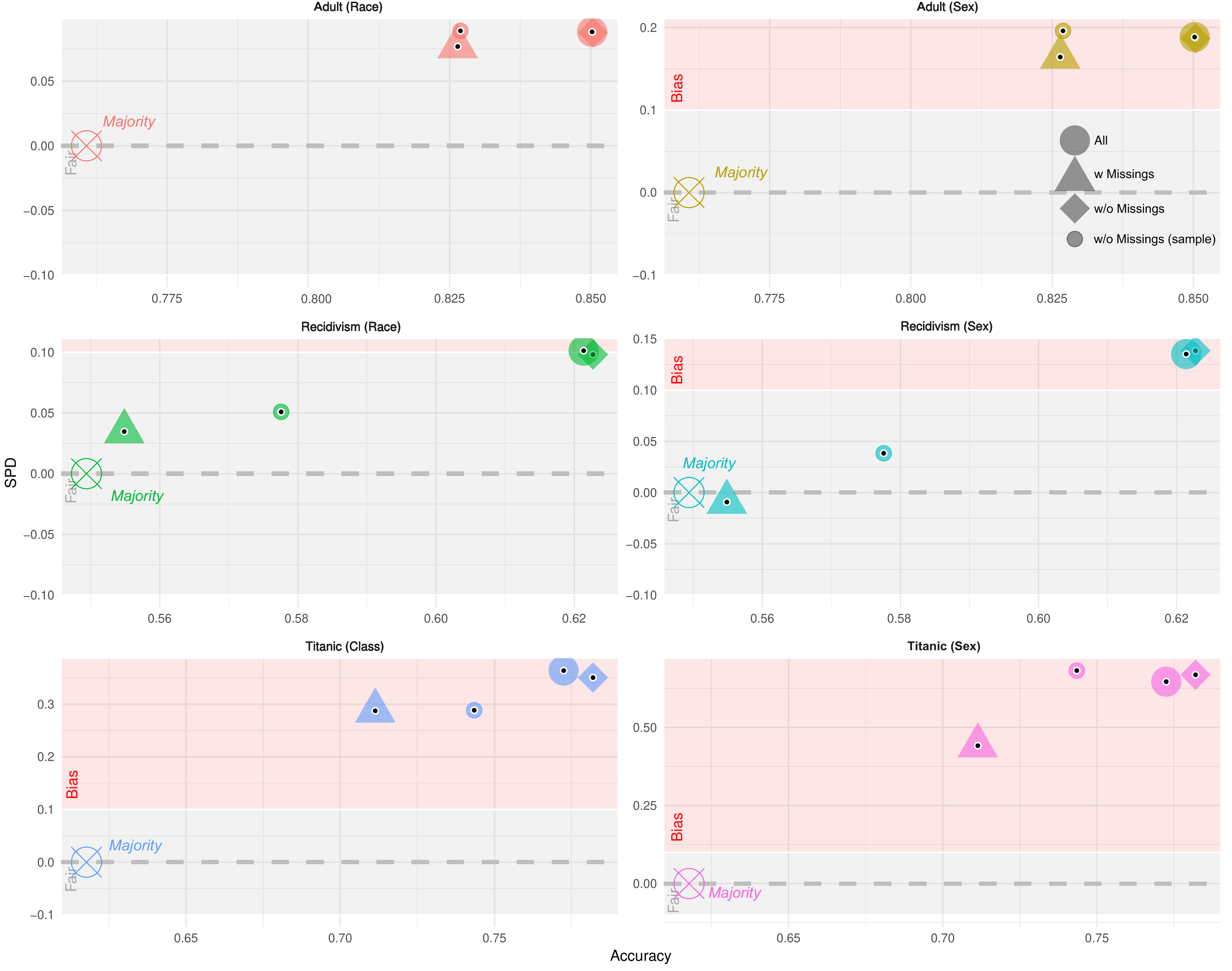}%
 \caption{Visualisation of the results from Table \ref{tab:IN} also including the majority class model. The \xaxis shows accuracy and the \yaxis shows SPD. The dashed grey line shows 0 SPD (no bias, perfect fairness), and the grey area goes between -0.1 to 0.1, a reasonably fair zone that helps see the magnitudes better. The stars are intended to denote statistical significance in a multiple pairwise-comparison between the means of the columns \emph{SPD (with \missing)}, \emph{SPD (w/o \missing)} and \emph{SPD (sample w/o \missing)}.
 }%
    \label{fig:IN}%
\end{figure}

\noindent In order to understand the effect of the missing values better, we are going to analyse the effect of those attributes that have missing values. Table \ref{tab:INcols} uses the same configuration as Table \ref{tab:IN}, except that the columns with at least one missing value were removed for training. The results are different, especially for Titanic, where bias is importantly reduced for Titanic and amplified for Sex, for all subsets and in comparison with the data and the model with all columns. Still, SPD for the subset with missing values is the best, as in the case for all columns. 
The surprise in this case comes when we look at accuracy, which is not strongly affected by removing the attributes (and in many cases the results are better).

\begin{table}[ht]
\centering
\resizebox{1\textwidth}{!}{%
\begin{tabular}{lC{1.5cm}C{2.5cm}C{2.75cm}C{2.5cm}C{2.85cm}C{2.5cm}C{2.75cm}C{2.5cm}C{2.5cm}}\toprule

\textbf{dataset} & \textbf{protected attribute} & \textbf{Acc (all rows)}  & \textbf{SPD (all rows)} &  \textbf{Acc (with \missing)}  & \textbf{SPD (with \missing)}  & \textbf{Acc (w/o \missing)}  & \textbf{SPD (w/o \missing)} & \textbf{Acc (sample w/o \missing)}  & \textbf{SPD (sample w/o \missing)} \\ 
  \hline
  
\multirow{ 2}{*}{Adult}  
& Race &  
0.8467 $\pm$ 0.0029 & \reduces0.0861 $\pm$ 0.0105 & 
0.8282 $\pm$ 0.0056 & \amplifies\textbf{0.0818 $\pm$ 0.0204*} & 
0.8463 $\pm$ 0.0030 & \reduces0.0861 $\pm$ 0.0105 & 
0.8260 $\pm$ 0.0055 & -0.0857 $\pm$ 0.0195\\  
& Sex & 
0.8467 $\pm$ 0.0029 & \reduces0.1854 $\pm$ 0.0072 & 
0.8282 $\pm$ 0.0056 & \amplifies\textbf{0.1605 $\pm$ 0.0203*} & 
0.8463 $\pm$ 0.0030 & \reduces0.1864 $\pm$ 0.0067 & 
0.8260 $\pm$ 0.0055 & 0.1943 $\pm$ 0.0182\\ \hline

\multirow{ 2}{*}{Recidivism} 
& Race  & 
0.6394 $\pm$ 0.0104 & \amplifies0.1250 $\pm$ 0.0299 & 
0.5333 $\pm$ 0.0218 & \reduces\textbf{0.0048 $\pm$ 0.0552*} & 
0.6374 $\pm$ 0.0103 & \amplifies0.1332 $\pm$ 0.0310 & 
0.5921 $\pm$ 0.0222 & 0.0563 $\pm$ 0.0514\\ 
& Sex  & 
0.6394 $\pm$ 0.0104 & \amplifies0.1393 $\pm$ 0.0374 & 
0.5333 $\pm$ 0.0218 & \reduces\textbf{-0.0105$\pm$ 0.0614*}  & 
0.6374 $\pm$ 0.0103 & \amplifies0.1448 $\pm$ 0.0358 & 
0.5921 $\pm$ 0.0222 & 0.0493 $\pm$ 0.0467 \\ \hline

\multirow{ 2}{*}{Titanic} 
& Class  & 
0.7862 $\pm$ 0.0179 & \reduces0.1666 $\pm$ 0.0634 & 
0.7478 $\pm$ 0.0266 & \reduces\textbf{0.1471 $\pm$ 0.0740*} & 
0.7827 $\pm$ 0.0180 & \reduces0.1758 $\pm$ 0.0646 & 
0.7665 $\pm$ 0.0268 & 0.2173 $\pm$ 0.1379\\
& Sex  & 
0.7862 $\pm$ 0.0179 & \amplifies0.9000 $\pm$ 0.0432 & 
0.7478 $\pm$ 0.0266 & \amplifies\textbf{0.7143 $\pm$ 0.1077*} & 
0.7827 $\pm$ 0.0180 & \amplifies0.8991 $\pm$ 0.0496 & 
0.7665 $\pm$ 0.0268 & 0.8769 $\pm$ 0.1284\\ \hline

\toprule
\end{tabular}
}
\caption{Results with the same configuration as Table \ref{tab:IN}, except that the columns with at least one missing value were removed for training (and hence not used by the tree during test either). Removed columns: Adult (\texttt{workclass}, \texttt{occupation}, \texttt{native\_country} ), Recidivsm (\texttt{days\_b\_screening\_arrest}, \texttt{c\_days\_from\_compas} and \texttt{c\_charge\_desc}) and Titanic (\texttt{age}, \texttt{fare} and \texttt{embarked}).}
\label{tab:INcols}
\end{table}

\comment{
\begin{figure}[ht]
	\centering
    \includegraphics[width=1\columnwidth]{Fairness_4ds_RemCols_rpart2_facet_100rep.pdf}%
 \caption{Result comparison from Table \ref{tab:INcols}. \todo{SPDs for recidivism have the wrong sign. Correct}}%
    \label{fig:INcols}%
\end{figure}
}

This suggests that including these rows, but treating the columns in a better way, could be beneficial. As we want to explore other machine learning methods, but most do not deal with missing values --and we must keep those rows--, we turn our analysis to imputation, in the next section.

\section{Treating missing values for fairness: delete or impute?}\label{sec:imputation}

Even if most libraries simply delete the rows with missing values, some imputation methods are so simple that it is difficult to understand why this option is not given in these packages (or included in the literature of fairness research). Fortunately, we can apply a preprocessing stage to every dataset with missing values where we can use any imputation method that we may have available externally. Nevertheless, imputation methods are not agnostic, and they can introduce bias as well, which can have an effect on fairness. As we saw in section \ref{sec:missing}, there are many reasons for missing values and some imputation methods (for instance those that impute using a predictive model) can amplify certain patterns, especially if the attributes with missing values are related to the class or the protected groups.  
In this first analysis we therefore prefer to use simple imputation methods, namely the one that replaces the missing value by the mean if it is a quantitative attribute and the mode if it is a qualitative attribute. 

Once the whole dataset has gone through this imputation method, we can now extend the number of machine learning methods that we can apply. We are using six machine learning techniques from the \emph{Caret} package \cite{kuhn2008building}: Logistic regression (LR), Naive Bayes (NB), Neural Network (NN), Random Forest (RF), the RPart decision tree (DT) and a support vector machine (SV) using a linear kernel. 

Figure \ref{fig:imputa} shows the results for the six cases we are considering in this paper with the six machine learning models, the majority class model (Majority
) model and a perfect model (Perfect). We have introduced this perfect model as a reference to enrich our analysis. Note that this is tantamount to what we did for Table 1, but for the test only (which explains why the values are roughly the same as those in the Table). These points are not achievable in practice (only theoretically), but help us to understand the bias that is already in the data, and how much this is amplified. We can compare these plots with the octagons (the space of possible classifiers, following proposition \ref{prop:octagon}) for each of the six combinations that can be shown as separate plots\footnote{We do not include the octagons in Figure \ref{fig:imputa} in order to keep the plots clean but also because the results in these plots are the average for several repetitions, while calculating means of octagons is not conceptually correct. Instead we show the octagons for the whole dataset in Appendix \ref{sec:proof}. With the separation, it is also easier to have in mind that the bounding octagons can have small variations for each of the partitions.} in Appendix \ref{sec:proof}. 
 This also reminds us that the perfect model is not unbiased. 

For each machine learning model in Figure \ref{fig:imputa} we show the results when trained by removing the rows with missing values (Deletion) and when trained after imputation (Imputation). 
Note that all models are evaluated with a test set that has applied imputation, otherwise we would not be able to compare all the options with the same data.

\begin{figure}[ht]
	\centering
    \includegraphics[width=1\columnwidth]{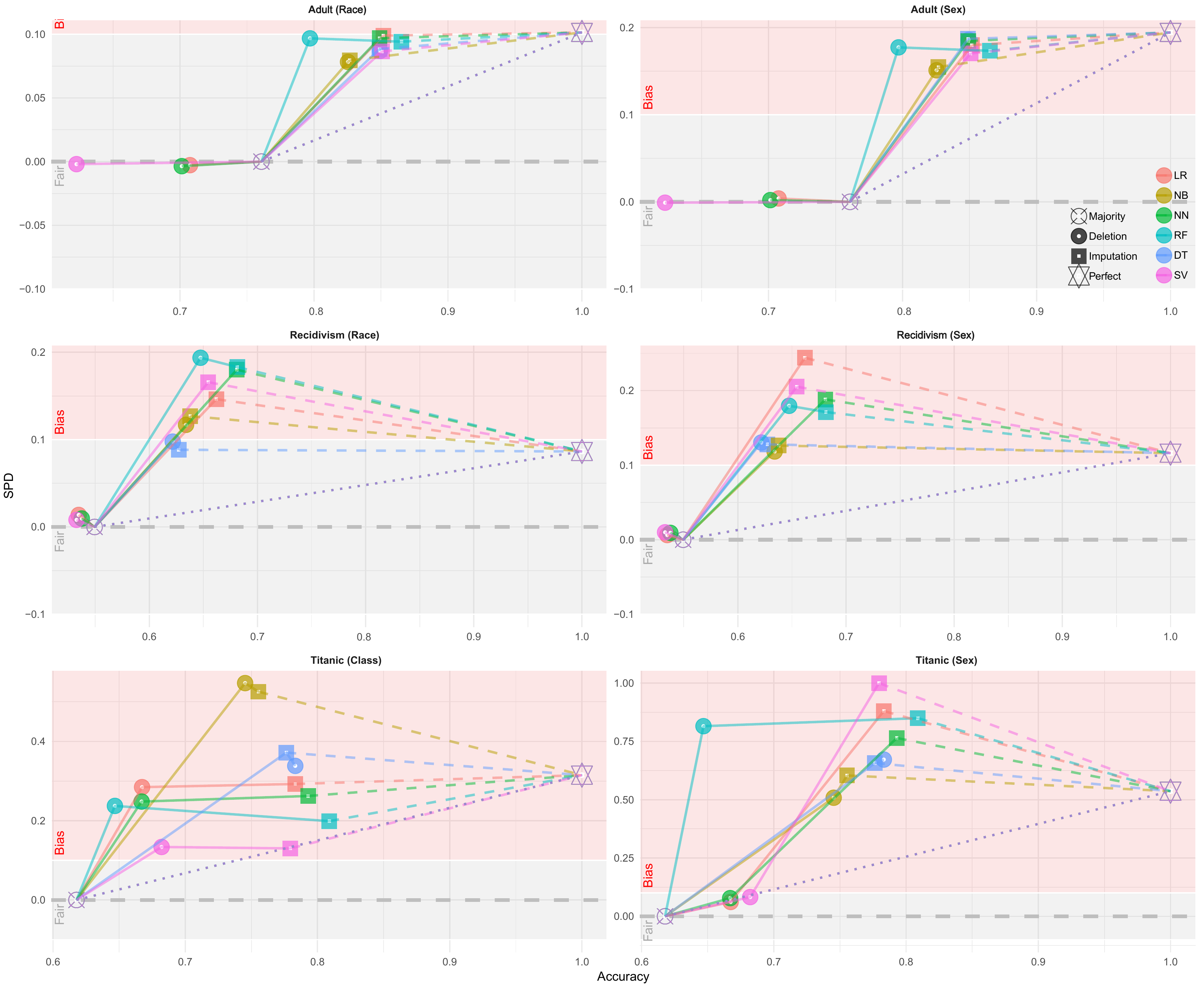}%
 \caption{Performance (accuracy) and Bias (SPD) for six machine learning models, using different subsets of data for training. Imputation: all data with the missing values being imputed. Deletion: only the data without the missing values). The test set is always with imputation. The plots also include two baselines: Majority (a majority class model) and Perfect (a perfect model, i.e., using the correct label from the test set). 
 }%
    \label{fig:imputa}%
\end{figure}

The first observation we can make is that in terms of accuracy, imputation is generally beneficial for all cases and almost all techniques. DT does not seem to get much from imputation, probably because using the mean or the mode for an attribute may strongly affect the distributions for the splits in those attributes. In terms of fairness, we see that for Adult all methods {\em reduce} bias from the perfect model, which is quite remarkable. An opposite result happens for Recidivism, and Titanic with Sex, while Titanic with Class is more mixed depending on the technique. When we look at the trade-off between fairness and performance, we see the expected trend of more performance implying less fairness, with the exception of Titanic with class. However, if look more carefully method by method, we see that RF behaves quite differently: using imputation increases accuracy significantly but does not amplify (and sometimes reduces) bias. Actually, RF is the only method that draws an almost straight line between the Deletion point, the Imputation point and the Perfect point.

If we look at the Pareto fronts (only built with the solid lines)
, SV and RF dominate for Adult; DT, LR and NN dominate for Recidivism with Race; RF for Recidivism with Race; SV and RF for Titanic with Class; and LR, SV, DT and RF for Titanic. Except for the latter case, the Pareto fronts are built exclusively from the results using the imputed dataset, not the dataset that deleted the rows with missing values. 

One thing that we can observe from all the plots is that most models are far (except  for SV at low accuracies in Titanic) from the segment joining the majority class model and the perfect model and hence very far from the boundary of ideal models that we showed in proposition \ref{prop:octagon}.

\section{Discussion}\label{sec:discussion}

We have presented the first comprehensive analysis of the effect of missing values on fairness. We investigated the causes of missing values and the causes of unfair decisions, and we saw many conceptual connections, or underlying causes for both. The surprising result was to find that, for the datasets studied in this paper, the examples with missing values seem to be fairer than the rest. Unfortunately, they are usually ignored in the literature, even if we can easily transform them from ugly to graceful if handled appropriately. Indeed, we have seen that they incorporate information that could be useful to limit the amplification of bias, but they can have a negative effect on accuracy too if not dealt with conveniently. As a result, it seems  that, for a good trade-off between performance and fairness, the imputation of the missing values instead of deleting them provides a wider range of results to choose from. 

While a predictive model can be trained on any particular subset of the data (e.g., excluding missing values), we need to evaluate any predictive model on {\em all} the examples (otherwise we would be cheating with the results and  the model would ultimate need to delegate these examples to humans). Consistently, if we want to use techniques that do not handle missing values --during deployment--, we need to delete the columns or use imputation methods. Still, we can use different techniques and learn the models without these rows, and see if they are better or worse than the models trained with the rows with imputation. This is exactly what we have done in the previous section and it is our recommendation as a general practice. Using these plots, we can locate the baseline classifiers and the space of possible classifiers, representing several techniques inside it. We can then have a clearer view of what to do for a particular problems, and which classifier we have to choose depending on the requirements of fairness (on the \yaxis) or performance (on the \xaxis).

This paper is just a first step in the analysis of the relation between missingness and fairness. Other imputation methods could be explored and they could be combined with mitigation methods. For this kind of extended study with so many new factors, we would need more datasets that are meaningful in terms of fairness issues and that contain missing values. We could also explore the generation of missing values, but in order to be meaningful, we would need to build a causal model and its relation to fairness. Similarly, there is a huge number of combinations of fairness metrics and performance metrics that could be explored. For instance, for very imbalanced datasets, it may be more interesting to at least compare the results with metrics such as AUC, and understand the space in this case. Note that this study should be accompanied by classifiers that predict scores or probabilities, jointly with a proper analysis of the relation between calibration and both fairness and missing values \cite{bella2010calibration}.

Of course more experimental analysis and new techniques could shed more light to the question. However, given the complexity of both phenomena (missing values and fairness) and their interaction, it is quite unlikely that a new (predictive) method of imputation works well for all techniques and all possible bias mitigation methods (and all the fairness and performance metrics). It is much more plausible that we need to analyse the results for a particular dataset, favourable class and protected groups, and find the best compromise on the Pareto front between the chosen performance and fairness metrics, as we have illustrated with this paper. In the end, fairness is a delicate issue for which context insensitive modelling pipelines may miss relevant information and dependencies for a particular problem. For the moment, we have learnt that ignoring the missing values is a rather ugly practice for fairness. 

\section*{Acknowledgements}

This material is based upon work supported by the EU (FEDER), and the Spanish MINECO under grant TIN 2015-69175-C4-1-R and RTI2018- 094403-B-C32, the Generalitat Valenciana PROMETEOII/2015/013 and PROMETEO/2019/098. F. Mart\'{\i}nez-Plumed was also supported by INCIBE (Ayudas para la excelencia de los equipos de investigaci\'on avanzada en ciberseguridad), the European Commission (Joint Research Centre) HUMAINT project (Expert Contract CT-EX2018D335821-101), and Universitat Polit\`ecnica de Val\`encia (Primeros Proyectos de lnvestigaci\'on PAID-06-18). J. Hernández-Orallo is also funded by an FLI grant RFP2-152.

\includepdf[pages=1-3]{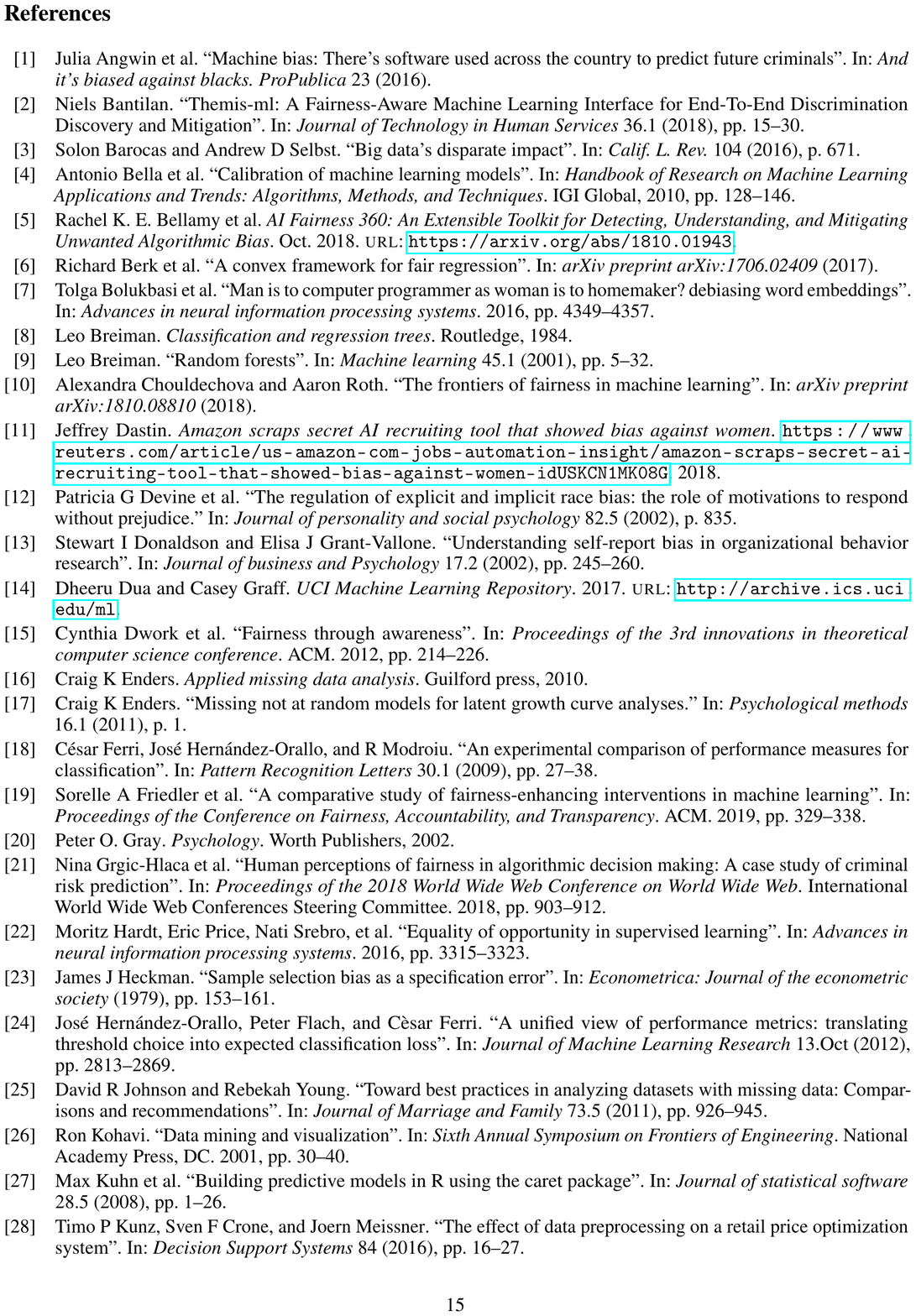}


\newpage
\appendix

\section{The SPD-accuracy  space}\label{sec:proof}

The following proposition shows the space for the possible classifiers in terms of the balance between SPD and Accuracy.

\begin{propositionnonumber} (Proposition \ref{prop:octagon} in the paper)
Assuming that the majority class is the favourable class, the set of classifiers with respect to Accuracy and SPD is bounded by the octagon composed of the following eight points:
\[\Acc = 1, \:\:\:\:\:\: \SPD^+_i = p^+(D_{X_i = a}) - p^+(D_{X_i \neq a}) \]



\[  \Acc =  1 - \frac{\Pos({D}_{X_i = a})}{|D|}, \:\:\:\:\:\:  \SPD^+_i =  - \frac{\Pos({D}_{X_i \neq a})}{|D_{X_i \neq a}|}
\]

\[ \Acc = \frac{\Neg({D}_{X_i = a}) + \Pos({D}_{X_i \neq a})}{|{D}|} , \:\:\:\:\:\:  \SPD^+_i  =  - 1 \]

\[\Acc = \frac{\Neg({D}_{X_i = a})}{|{D}|} , \:\:\:\:\:\: \SPD^+_i = - 1 + \frac{  \Pos({D}_{X_i \neq a})}{|D_{X_i \neq a}|} \]

\[\Acc = 0 , \:\:\:\:\:\: \SPD^+_i =  p^+(D_{X_i \neq a}) - p^+(D_{X_i = a})\]

\[  \Acc =  \frac{\Pos({D}_{X_i = a})}{|D|}, \:\:\:\:\:\:  \SPD^+_i =   \frac{\Pos({D}_{X_i \neq a})}{|D_{X_i \neq a}|}
\]

\[ \Acc = \frac{\Pos({D}_{X_i = a}) + \Neg({D}_{X_i \neq a})}{|{D}|} , \:\:\:\:\:\:  \SPD^+_i  =  1 \]

\[ \Acc = 1 - \frac{\Neg({D}_{X_i = a})}{|{D}|} , \:\:\:\:\:\:  \SPD^+_i  =  
1 - \frac{\Pos({D}_{X_i \neq a})}{|D_{X_i \neq a}|}
 \]


\end{propositionnonumber}

\begin{proof}
As we have assumed that the majority class is the favourable class, we have the following point for the perfect model:


\[\Acc[\mbox{Perfect}] = 1 \]
\[\SPD^+_i[\mbox{Perfect}] = p^+(D_{X_i = a}) - p^+(D_{X_i \neq a}) \]




We denote by $\hat{D}$ the data labelled with a classifier, so it will have accuracy and SPD as follows.

\[\Acc(\hat{D}) = \frac{\TP(\hat{D}) + \TN(\hat{D})}{|{D}|}
=  \frac{\TP(\hat{D}_{X_i = a}) + \TN(\hat{D}_{X_i = a}) + \TP(\hat{D}_{X_i \neq a}) + \TN(\hat{D}_{X_i \neq a})}{|{D}|} 
\]

\begin{eqnarray*}
\SPD^+_i(\hat{D}) & = & p^+(\hat{D}_{X_i = a}) - p^+(\hat{D}_{X_i \neq a}) = \frac{\TP(\hat{D}_{X_i = a}) + \FP(\hat{D}_{X_i = a})}{|D_{X_i = a}|} - \frac{\TP(\hat{D}_{X_i \neq a}) + \FP(\hat{D}_{X_i \neq a})}{|D_{X_i \neq a}|} 
\end{eqnarray*}







Let us start with the perfect classifier, and let us assume wlog that $\SPD^+_i$ is positive. 

{\bf SPD positive}

We consider that the favourable class is the positive class, so we have four cases in which a prediction can change:

\begin{enumerate}
    \item In $D_{X_i = a}$, a TN changes to a FP. This  increments $\SPD^+_i$ in $\frac{1}{|D_{X_i = a}|}$.
        \item In $D_{X_i = a}$, a TP changes to a FN. This  decrements $\SPD^+_i$ in $\frac{1}{|D_{X_i = a}|}$.
            \item In $D_{X_i\neq a}$, a TN changes to a FP. This  decrements $\SPD^+_i$ in $\frac{1}{|D_{X_i \neq a}|}$.
        \item In $D_{X_i \neq a}$, a TP changes to a FN. This  increments $\SPD^+_i$ in $\frac{1}{|D_{X_i \neq a}|}$.
\end{enumerate}
 
As $\SPD^+_i$ is positive and we want to reduce it, then only the mistakes 2 and 3 are useful. If  $|D_{X_i = a}| < |D_{X_i \neq a}|$ then it is more advantageous to do 2. Otherwise, it is more advantageous to do 3 first.

{\bf SPD positive and $|D_{X_i = a}| < |D_{X_i \neq a}|$}

In this situation the strategy consists in moving one by one all the elements from TP to FN (exactly all the TP instances in ${D}_{X_i = a}$, which are simply the positives as we started with the perfect classifier), so that we move to a point that is located at:

\[  \Acc =  \Acc[Perfect] - \frac{\Pos({D}_{X_i = a})}{|D|} = 1 - \frac{\Pos({D}_{X_i = a})}{|D|} \]

\[  \SPD^+_i =  \SPD^+_i [Perfect] - \frac{  \Pos({D}_{X_i = a})}{|D_{X_i = a}|} 
= - \frac{\Pos({D}_{X_i \neq a})}{|D_{X_i \neq a}|}
\]

As we did not have FP nor FN, because we started from the perfect classifier, and all the TP for $X=a$ have been converted into FN, we only have TN and FN for $X=a$, which always predicts the negative class when $X=a$ and always correct when $X \neq a$. 

Then, when case 2 is exhausted, we can start with case 3 and move one by one all the elements from TN to FP (exactly all the original TN instances in ${D}_{X_i \neq a}$, which are the negatives), and we get the point:

\begin{eqnarray*}
\Acc & = &  \Acc[Perfect] - \frac{\Pos({D}_{X_i = a}) + \Neg({D}_{X_i \neq a})}{|D|} =  1 -  \frac{\Pos({D}_{X_i = a}) + \Neg({D}_{X_i \neq a})}{|D|}  \\
& = &  \frac{\Neg({D}_{X_i = a}) + \Pos({D}_{X_i \neq a})}{|{D}|} \end{eqnarray*}

\begin{eqnarray*}
  \SPD^+_i & = & \SPD^+_i [Perfect] - \frac{  \Pos({D}_{X_i = a})}{|D_{X_i = a}|} - \frac{  \Neg({D}_{X_i \neq a})}{|D_{X_i \neq a}|} \\
  & = & - \frac{\Pos({D}_{X_i \neq a})}{|D_{X_i \neq a}|} - \frac{  \Neg({D}_{X_i \neq a})}{|D_{X_i \neq a}|} \\
  & = & - 1
\end{eqnarray*}

Again, the last step is explained as we did not have FP nor FN, because we started from the perfect classifier, and all the TP for $X=a$ have been converted into FN, so we only have TN and FN for $X=a$, which means it always predicts the negative class when $X=a$. As all the TN for $X \neq a$ have been converted into FP, we only have TP and FP for $X \neq a$, which means it always predicts the positive class when $X\neq a$. This is the most negative bias, $-1$.

Note that as the value of SPD is negative already after all the 2s, we know that perfect fairness for SPD has been achieved before exhausting the 2s. 
However, the points on the segments connecting these points and the original perfect classifier do not need to contain the majority classifier, which will usually be left out of these two segments.

And now if we want to explore the other two cases, we are at $\SPD=-1$ and we can only increase. If  $|D_{X_i = a}| < |D_{X_i \neq a}|$ then now we want to increase $\SPD$ as slowly as possible to cover the whole space, so it is now more advantageous to do 4 as the steps are smaller.

As in the previous cases, now we change TP to FN, so it is the positives in $D_{X_i \neq a}$ that we lose, so incrementally over the previous point this would lead to:

\begin{eqnarray*}
\Acc & = &  \frac{\Neg({D}_{X_i = a}) + \Pos({D}_{X_i \neq a}) - \Pos({D}_{X_i \neq a})}{|{D}|} \\
  & = & \frac{\Neg({D}_{X_i = a})}{|{D}|} 
  \end{eqnarray*}

\begin{eqnarray*}
  \SPD^+_i  & = & - 1 + \frac{  \Pos({D}_{X_i \neq a})}{|D_{X_i \neq a}|}
\end{eqnarray*}

The other four points of the octagon are mirrored with the above (symmetric with the limits of \Acc and \SPD respectively).

{\bf SPD positive and $|D_{X_i = a}| > |D_{X_i \neq a}|$}

The situation when  $|D_{X_i = a}| > |D_{X_i \neq a}|$ follows similarly but does 2 before 3 and then gets $D_{X_i \neg a}$ and $D_{X_i = a}$ swapped. 
Similarly for the rest of the octagon. Basically, we are exploring a convex octagon, so making a wrong choice does not give us the whole space (there would be concave parts).

{\bf SPD negative}
Finally, the scenario when SPD of the perfect classifier is negative is analogous.
\end{proof}

An example of the octagonal SPD-Accuracy space for a test dataset is included in Figure \ref{fig:octa}. The dataset presents the following metrics: $ \SPD^+_i= 0.3$,  $p^+(D_{X_i = a})=0.8$,  $p^+(D_{X_i \neq a})=0.5$ and $|D_{X_i = a}| /|D|=0.3$. Figure \ref{fig:octa_datasets} shows the corresponding SPD-Accuracy spaces for the six datasets and protected groups included in Table \ref{tab:PRE}.

\begin{figure}[ht]
	\centering
    \includegraphics[width=0.8\columnwidth]{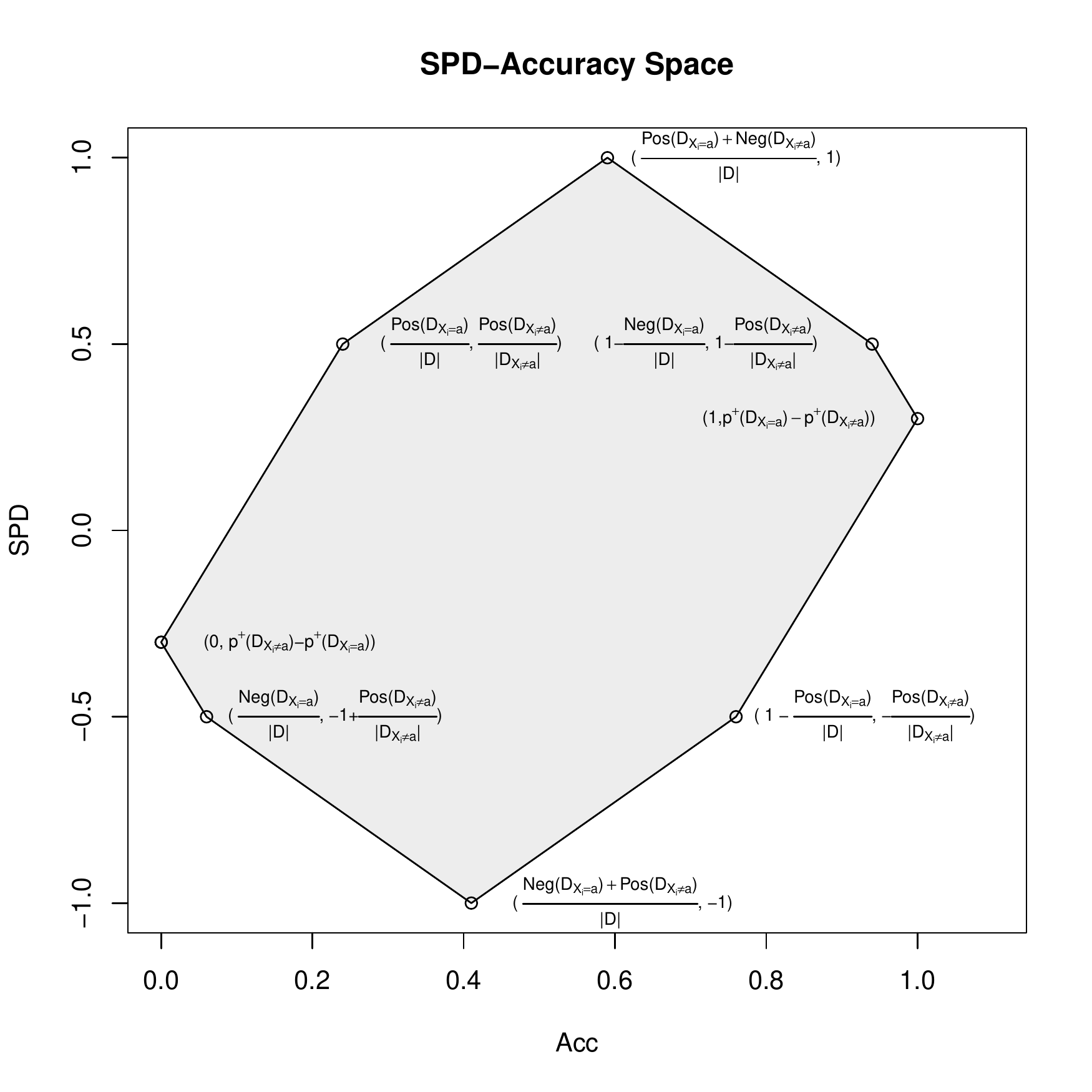}%
 \caption{SPD-accuracy space for a dataset with $ \SPD^+_i= 0.3$,  $p^+(D_{X_i = a})=0.8$,  $p^+(D_{X_i \neq a})=0.5$ and $|D_{X_i = a}| /|D|=0.3$.
 }%
    \label{fig:octa}%
\end{figure}

\begin{figure}[ht]
	\centering
    \includegraphics[width=0.45\columnwidth]{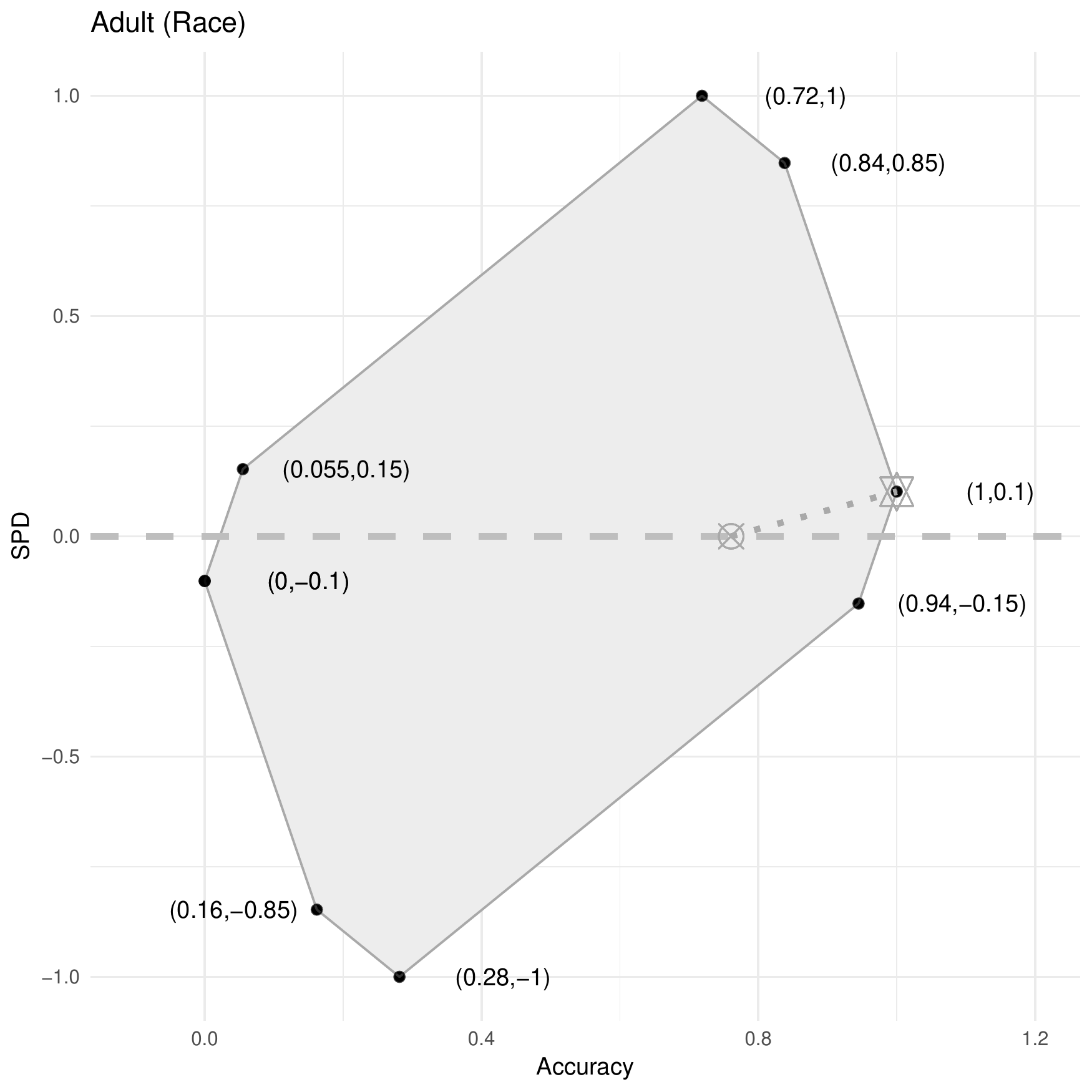}
    \includegraphics[width=0.45\columnwidth]{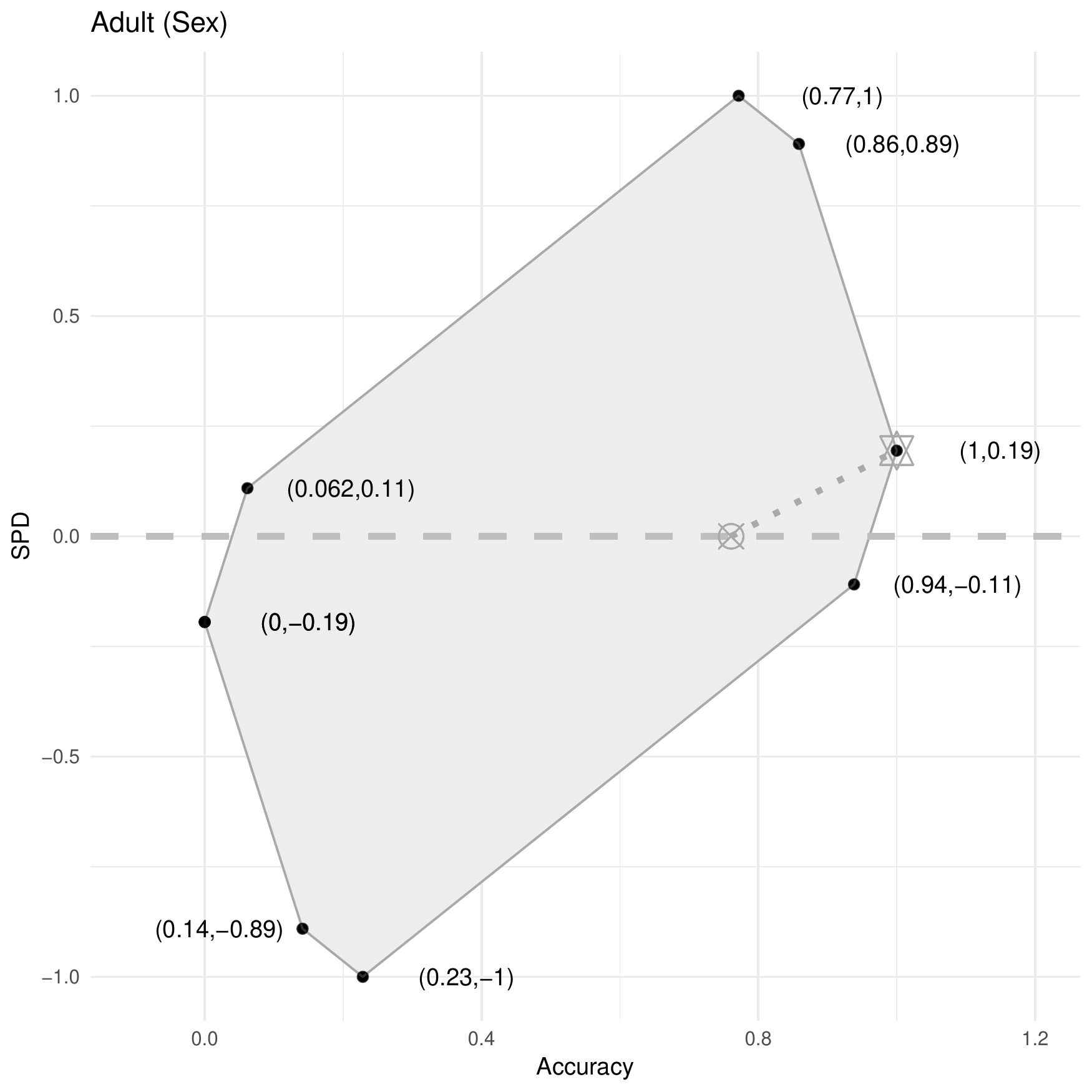}
    \includegraphics[width=0.45\columnwidth]{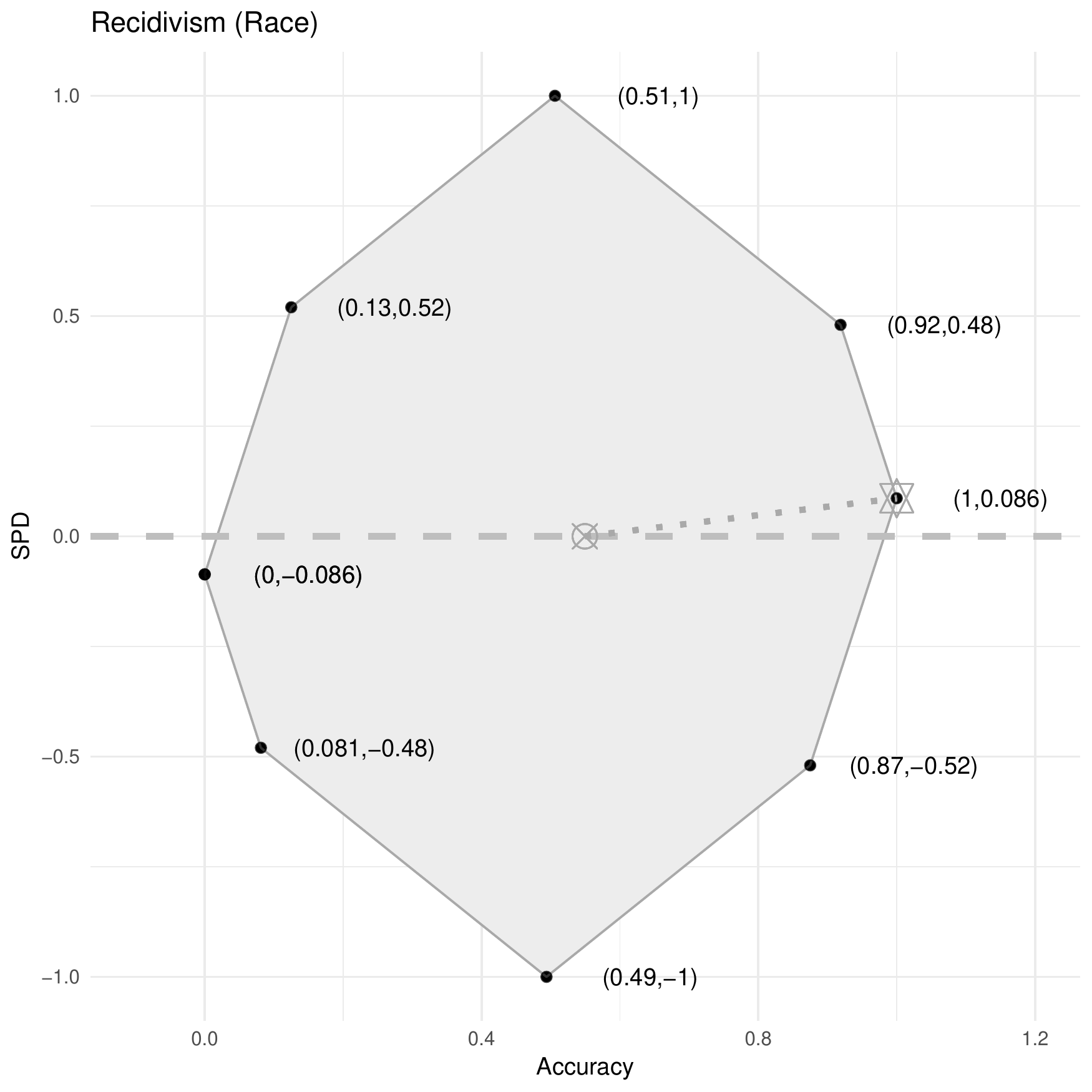}
    \includegraphics[width=0.45\columnwidth]{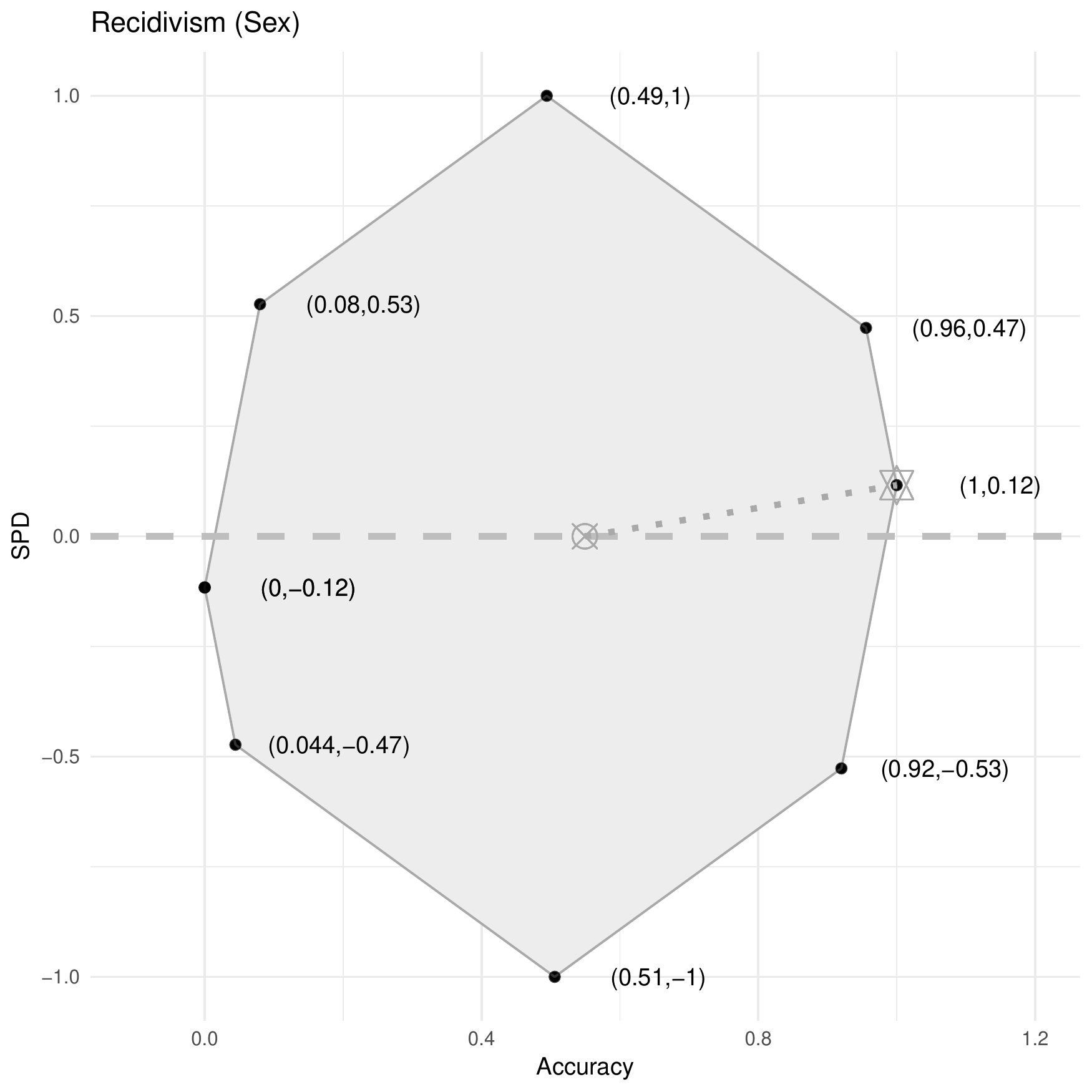}
    \includegraphics[width=0.45\columnwidth]{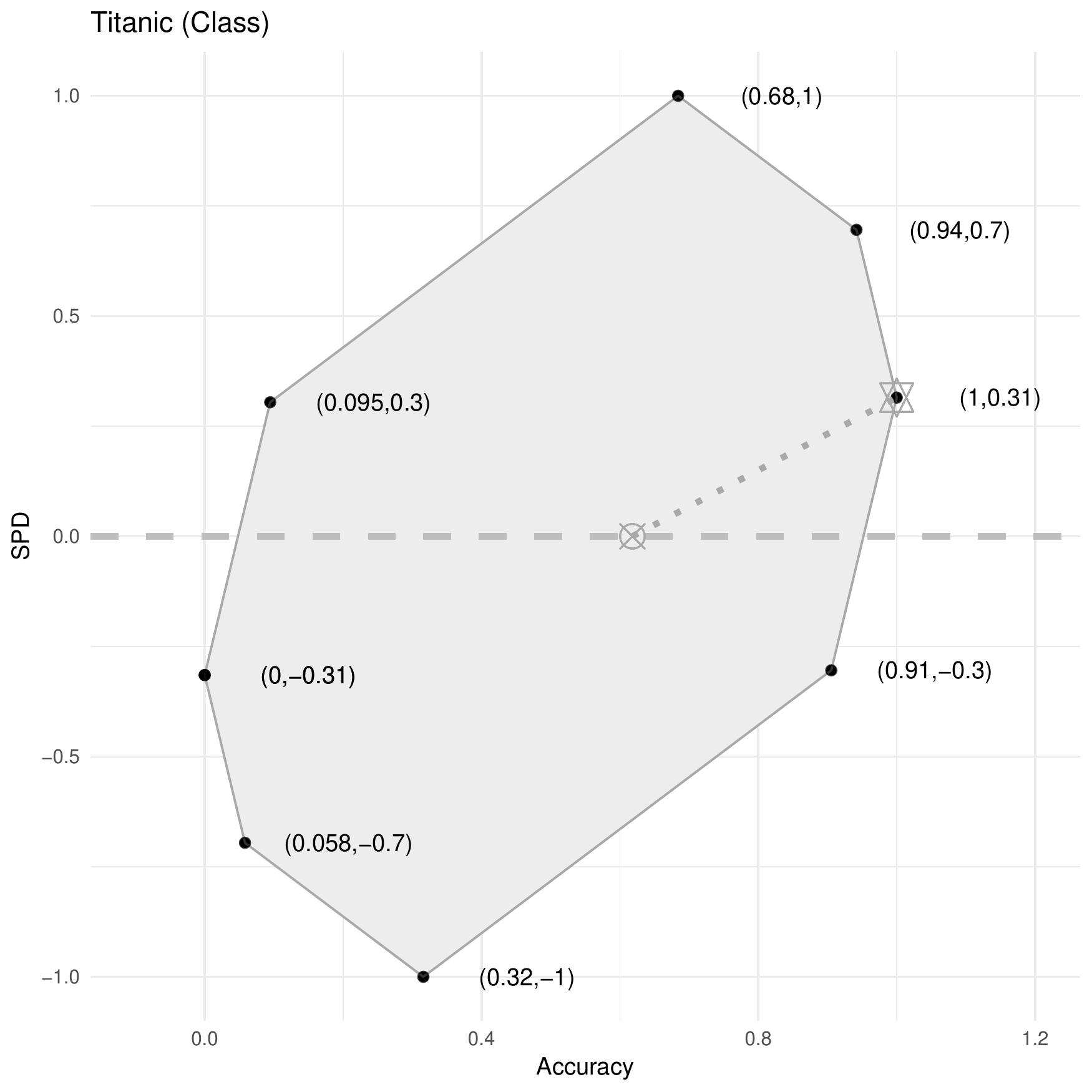}
    \includegraphics[width=0.45\columnwidth]{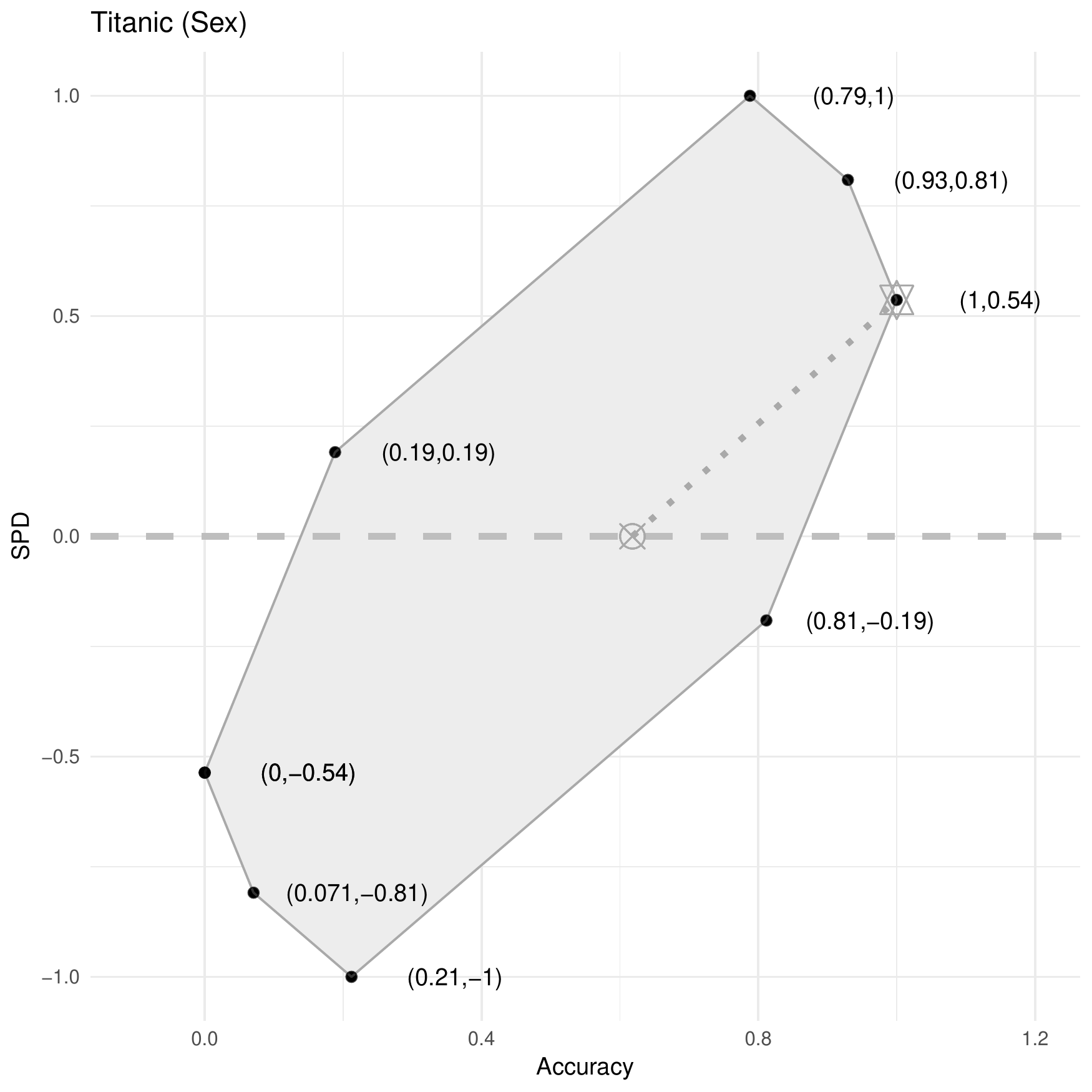}
 \caption{SPD-accuracy spaces for the datasets in Table \ref{tab:PRE}. We also show the majority class classifier with a cross, and connect it to the perfect classifier (represented by a star).
 }%
    \label{fig:octa_datasets}%
\end{figure}

\end{document}